\documentclass{article}
\PassOptionsToPackage{numbers, compress}{natbib}
\usepackage[final]{neurips_2021}
\usepackage[pagebackref=true,breaklinks=true,letterpaper=true,colorlinks,bookmarks=false]{hyperref}
\usepackage[american]{babel}
\usepackage{mathtools} % amsmath with fixes and additions
\usepackage{wrapfig} 
\usepackage{booktabs} % commands to create good-looking tables
\usepackage{tikz} % nice language for creating drawings and diagrams
\usepackage{amsfonts}
\usepackage{amsthm}
\usepackage[ruled]{algorithm2e}
\newtheorem{theorem}{Theorem}
\usepackage{tikz}
\usetikzlibrary{bayesnet, matrix}

\newcommand{\Ex}[2]{\mathbb{E}_{#1}\left[#2\right]}
\newcommand{\R}{\mathbb{R}}
\newcommand{\diag}[0]{\operatorname{diag}}
\newcommand{\perm}[0]{\operatorname{perm}}
\newcommand{\dkl}{D_\text{KL}}
\newcommand{\name}{BCD Nets}
\newcommand{\fullname}{Bayesian Causal Discovery Nets}

\title{
  BCD Nets: Scalable Variational Approaches for Bayesian Causal Discovery
}

\author{%
  Chris Cundy$^1$ \quad Aditya Grover$^{2,3}$ \quad Stefano Ermon$^1$\vspace{2pt}\\
  $^1$Department of Computer Science, Stanford University\\
  $^2$Facebook AI Research\\
  $^3$University of California, Los Angeles\vspace{4pt}\\
  \texttt{\{cundy, ermon\}@cs.stanford.edu}\\
	\texttt{adityag@cs.ucla.edu} \quad
}

\begin{document}
\maketitle

\begin{abstract}
A structural equation model (SEM) is an effective framework to reason over causal relationships represented via a directed acyclic graph (DAG).
Recent advances have enabled effective maximum-likelihood point estimation of DAGs from observational data. 
However, a point estimate may not accurately capture the uncertainty in inferring the underlying graph in practical scenarios, wherein the true DAG is non-identifiable and/or the observed dataset is limited.
We propose Bayesian Causal Discovery Nets (BCD Nets), a variational inference framework for estimating a \textit{distribution} over DAGs characterizing a linear-Gaussian SEM.
Developing a full Bayesian posterior over DAGs is challenging due to the the discrete and combinatorial nature of graphs.
We analyse key design choices for scalable VI over DAGs, such as 1) the parametrization of DAGs via an expressive variational family, 2) a continuous relaxation that enables low-variance stochastic optimization, and 3) suitable priors over the latent variables.
We provide a series of experiments on real and synthetic data showing that BCD Nets outperform maximum-likelihood methods on standard causal discovery metrics such as structural Hamming distance in low data regimes. 
\end{abstract}
\section{Introduction}\label{sec:intro}
One of the key uses of statistical methods is learning causal relationships from observed data a.k.a. \textit{causal discovery}~\citep{peters2014causal}. Causal models allow us to forecast the effects of interventions and counterfactuals in several real-world domains, such as economic policy \citep{varian2016causal} and medicine \citep{shalitEstimatingIndividualTreatment2016}.
Although early approaches to statistical inference
emphasised that `correlation is not causation' \citep{fisher1958lung}, it has since been shown that for certain families of data-generating processes, it is indeed possible to infer causal relationships from purely observational data \citep{pearl2009causality, little2000causal}.
One such widely studied data-generating process is the linear-Gaussian structural equation model (SEM) \citep{pearl2009causality}, where the causal relationships between the random variables in the model can be represented via a weighted directed acyclic graph (DAG).
The value of any variable in the DAG of a linear-Gaussian SEM is given by a linear combination of the values of its parent nodes and additive noise.
For causal discovery, naive Monte-Carlo sampling or enumeration of the possible DAGs quickly becomes intractable, since the number of possible DAGs over a model grows super-exponentially with the number of variables \citep{friedman2003being}.
A variety of methods have been developed over the years to efficiently sample or optimize over DAGs \citep{cussens2012bayesian, teyssierOrderingbasedSearchSimple2005, spirtes1991algorithm, scanagatta2015learning}. 
For example, a recent line of work scales to high dimensions by maximizing the likelihood of the model (MLE)
over a set of continuous relaxations of adjacency matrices using gradient-based methods and specialized DAG regularization terms.
\citep{zhengDAGsNOTEARS2018, yuDAGGNNDAGStructure2019, ngRoleSparsityDAG2020, wei2020dags}.

However, the majority of the aforementioned works for scalable causal discovery in linear-Gaussian SEMs focus on recovering \emph{point estimates} for the underlying DAG via MLE.
In many practical scenarios, however, a point estimate 
fails to reflect the uncertainty in inferring the underlying DAG.
This includes scenarios where the true DAG is non-identifiable given (infinite) observational data as well as 
practical limitations due 
to an imperfect optimization algorithm, model mismatch, or simply a finite dataset.
In any of the above scenarios, it is desirable to obtain an explicit \emph{posterior distribution} over the unobserved DAG instead of a single point estimate~\citep{gelmanBayesianDataAnalysis2013}.
Causal inference is increasingly being applied in situations with important real-world consequences, where such a Bayesian estimation procedure could be useful to sample and reason over alternative generative mechanisms for observed data.

To achieve this goal, we propose \fullname{} (\name{}), an algorithmic framework for Bayesian causal discovery in linear-Gaussian SEMs based on modern variational inference ~\citep{blei2017variational}. 
Classical approaches to Bayesian causal discovery struggle in higher dimensions \citep{heckerman1999bayesian,fragoso2018bayesian}, with limited improvements via Monte Carlo approximations \citep{han2017efficient,viinikka2020towards}.
In order to scale our variational method to high dimensions, we address several design challenges.
First, we describe an expressive variational family of factorized posterior distributions over the SEM parameters (edge weights and noise variance) using deep neural networks.
The factorization exploits the decomposition of DAGs into triangular matrices and permutations for specifying the distribution over edge weights and node orderings respectively. 
Further, for low-variance stochastic optimization of the variational objective, we exploit recent advances in modeling and reparametrizing distributions over permutations via continuous relaxations \citep{menaLearningLatentPermutations2018, li2021discovering}.
Finally, we employ a horseshoe prior~\citep{carvalho2009handling} on the edge weights, which promotes sparsity.\footnote{Code is available at \url{github.com/ermongroup/BCD-Nets}}

We evaluate \name{} for causal discovery on a range of synthetic and real-world data benchmarks. 
Our experiments demonstrate that with finite datasets, there is considerable uncertainty in the inferred posterior over DAGs.
Using \name{}, we are able to effectively quantify the uncertainty and significantly outperform competing estimators~\citep{ngRoleSparsityDAG2020,viinikka2020towards} on the standard Structured Hamming Distance metric, especially in low data regimes.

\section{Preliminaries}\label{sec:setting-3}
\subsection{Linear-Gaussian Structural Equation Models}

A structural equation model (SEM) is a collection of random variables \(x_1, \ldots, x_d\) associated with a directed acyclic graph (DAG) \(G\) with \(d\) nodes~\citep{pearl2009causality}.
The SEM consists of a series of equations \(x_j = f_j(\operatorname{Pa}(x_j), \epsilon_j)\), where \(\operatorname{Pa}(x_j)\) gives the values of the parents of the \(j\)th node in \(G\) and \(\epsilon_j\)
is a noise variable.
For a linear-Gaussian SEM, the equations $f_j$ are linear and the noise $\epsilon_j$ is additive Gaussian. Considering \(X = [x_1, \ldots, x_d]\) as a vector in \(\R^d\) and \(W\) as the weighted adjacency matrix of \(G\),
we can see that \(X\) must satisfy \(X = W^\top X + \epsilon\),
where \(\epsilon \sim \mathcal{N}(0, \Sigma)\) is the additive noise vector and \(\Sigma = \diag\left\{\sigma_1^2, \ldots, \sigma_d^2 \right\}\) is a diagonal noise covariance matrix. 
We will denote the setting when all noise variances are equal i.e., \(\sigma_1=\ldots=\sigma_d=\sigma\) as an ``equal variance'' setting and ``non-equal variance'' otherwise.
Figure~\ref{fig:causal_graph} shows an illustration.
Linear SEMs have been used to model microarray data~\citep{peters2014identifiability} and protein pathways~\citep{eigenmann2017structure}, among many other systems.

\begin{figure}[t]
\centering

 \begin{tikzpicture}

  \node[circle] (A) at (-8,-2) [draw, minimum width=0.5cm,minimum height=0.5cm] {$x_1$};
  
  \node[circle] (B) at (-6.5,-2) [draw, minimum width=0.5cm,minimum height=0.5cm] {$x_2$};
  
  \node[circle] (C) at (-6.5,-0.5) [draw, minimum width=0.5cm,minimum height=0.5cm] {$x_3$};
  
   \node[circle] (D) at (-5, -2) [draw, minimum width=0.5cm,minimum height=0.5cm] {$x_4$};

   \matrix [matrix of math nodes,left delimiter=(,right delimiter=)](E)  at (-2.75, -1.5){ 
    x_1 \\
    x_2 \\
    x_3 \\
    x_4 \\
   };
   \node[] (G) at (-1.75, -1.5) {=};
   \matrix [matrix of math nodes,left delimiter=(,right delimiter=)](F)  at (0, -1.5){ 
    0 & 0 & 0 & 0 \\
    1.3 & 0 & 0 & 0 \\
    0.3 & 2.5 & 0 & 0 \\
    0 & 5.7 & 0 & 0 \\
   };
 \matrix [matrix of math nodes,left delimiter=(,right delimiter=)](H)  at (2.25, -1.5){ 
    x_1 \\
    x_2 \\
    x_3 \\
    x_4 \\
   };
    \node[] (I) at (3.125, -1.5) {+};
    \matrix [matrix of math nodes,left delimiter=(,right delimiter=)](I)  at (4, -1.5){ 
    \epsilon_1 \\
    \epsilon_2 \\
    \epsilon_3 \\
    \epsilon_4 \\
   };
   
   % \node [] (N) at (3.5, -3) {$\epsilon \sim \mathcal{N}(0, \Sigma)$.}; 
%   \node[] (E) at (0, -2) [draw] {
%   \begin{equation*}
%      A =  \begin{bmatrix}
% a & b & c \\
% d & e & f \\
% g & h & i
% \end{bmatrix}
%   \end{equation*}
%   };
   
    \foreach \from/\to in {A/B, B/C, B/D, A/C}
\draw [->] (\from) -- (\to);

%  \node[circle] (d) at (1,1) [draw, minimum width=0.5cm,minimum height=0.5cm] {$\mathbf{z}$};
%  \node[draw=none,fill=none] (e) at (2,0) {$\theta$};
%  \node[rectangle] (f) at (2,1) [draw, minimum width=0.5cm,minimum height=0.5cm] {$f$};
%  \foreach \from/\to in {c/d, d/f, e/f}
% \draw [->] (\from) -- (\to);
% \draw [->] (c) to [out=60,in=120] (f);
  \end{tikzpicture}
  \caption{A Structural Equation Model (SEM). \textbf{Left:} DAG with 4 nodes. \textbf{Right:} Linear-Gaussian SEM (with \(\epsilon \sim \mathcal{N}(0, \Sigma)\))
%  \kristy{maybe just to be extra clear you can say $x_2$ and $x_1$ are parents of $x_3$, etc, but feel free to ignore}
 }
\label{fig:causal_graph}
% \vspace{-0.15in}
\end{figure}
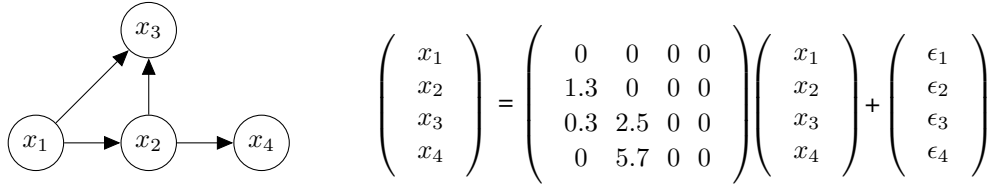

\subsection{Causal Discovery}
Given some dataset $X^n_1 = \{X_1, X_2, \dots, X_n\}$ drawn i.i.d.\ from a linear-Gaussian SEM, we are interested in inferring
the adjacency matrix $W$ and the diagonal entries of the covariance matrix $\Sigma$. 

\paragraph{Maximum Likelihood Estimation (MLE)}
This approach obtains a point estimate for $W$ and $\Sigma$ by maximizing the likelihood of the dataset \(X^n_1\).
Using the fact that $W$ is constrained to be a DAG, it can be shown that \(I - W^\top\) is always invertible.\footnote{Writing \(W = PLP^\top\) from Section~\ref{sec:method}, \(I - W^\top = (I - PL^\top P^\top) = P(I - L^\top)P^\top\). Now $(I - L^\top)$ is upper-triangular with unit diagonal, so $\det(I - L^\top) = 1$.  Thus $(I - W^\top)$ has unit determinant, so is invertible.
}.
Hence, we can rearrange to obtain \(X = {(I - W^\top)}^{-1}\epsilon\), obtaining \(X\) as the product of a matrix and a Gaussian vector.
It follows that \(X\) itself has a Gaussian distribution, \(X \sim \mathcal{N}(0, \Theta^{-1})\), with precision matrix \(\Theta = (I - W)\Sigma^{-1}{(I - W)}^\top\).
This gives a joint log-likelihood over the dataset \(X^n_1\) as
\begin{align}
  \label{eq:3}
  \log p(X^n_1 ; \Sigma, W) = \frac{n}{2}\left(\log \det \Theta - d \log (2\pi)\right)- \frac12 \sum_{i=1}^n X_{i} ^\top \Theta X_i,
\end{align}
where \(X_i\) runs over the \(n\) points in the dataset and \(\Theta\) is given above. 
The space of DAG adjacency matrices  $W$ is characterised by the acyclicity constraint. When viewed as an subset of the space of all adjacency matrices (\(\R^{d\times d}\)), this is a piecewise linear manifold with a number of facets that grows super-exponentially in the dimension (roughly as \(d!2^{d^2/2}\)), rendering any approach based on enumeration over all DAGs intractable.
Many approaches have been proposed to scale causal discovery via MLE to high-dimensional data.
This includes recent approaches
which relax the DAG constraint to a larger, continuous set of adjacency matrices, such as \(\R^{d\times d}\) in~\cite{zhengDAGsNOTEARS2018} or the Birkhoff polytope of doubly-stochastic matrices in~\cite{birdal2019probabilistic}, which enables the use of gradient-based optimization.

However, there are a number of challenges with using point estimators for causal discovery.
Fundamentally, the true SEM parameters are identifiable from observational data only under specific conditions.
In fact, the map between the the data distribution parameters \(\Theta\) and the SEM parameters \(\left\{W, \Sigma\right\}\) may not be bijective even in the limit of infinite data and an oracle optimizer.
Theoretical results on identifiability in linear-Gaussian SEMs currently hold only under restricted settings; notably, this only includes 
the equal variance setting~\citep{peters2014identifiability}.
Further, even if the noise variances are equal in the ground-truth SEM, we are limited in practice by the finite size of the dataset and MLE could still converge to an incorrect solution. 
Finally, such point estimators are typically not robust to potential misspecifications in the likelihood model, e.g., a non-Gaussian noise model, non-linear cause-effect relationships, etc.
For the above scenarios, it is useful to characterize the uncertainty in estimating the SEM parameters to guide downstream analysis. For example, in medical applications giving a distribution of possible causal pathways, with quantified uncertainty, could be much more useful than a point estimate of the most likely pathway.  

\paragraph{Bayesian Estimation.} In contrast to point estimators, Bayesian methods explicitly characterize the uncertainty in the estimated parameters~\citep{gelmanBayesianDataAnalysis2013}.
That is, we treat the unknown SEM parameters \(\left\{W, \Sigma\right\}\) as random variables associated with a prior distribution \(p(W, \Sigma)\).
The likelihood model $p(X^n_1 \vert \Sigma, W)$ follows the same expression as the RHS in equation~\eqref{eq:3}.
Given the prior and the likelihood, we obtain a posterior distribution \(p(W, \Sigma \vert X_1^n)\) over \(\left\{W, \Sigma\right\}\) via Bayes rule, which quantifies the uncertainty in estimating the SEM parameters. 
As long as the likelihood model is well-specified and the prior includes the ground-truth SEM parameters in its support, as the dataset size increases the posterior will concentrate around a set of SEM parameters. In the equal variance case these will be the true SEM parameters. In the non-equal variance case the posterior
will concentrate on the set of DAG parameters quasi-equivalent to the ground truth as defined in~\cite{ngRoleSparsityDAG2020}. These are the set of DAG parameters which generate data with the same covariance as the ground truth, and are hence indistinguishable based on data alone.
We discuss this further in the appendix, Section~\ref{sec:recov-ground-truth}. 

The key challenge in Bayesian estimation is tractable computation of the posterior distribution in high-dimensional spaces. With the exception of specific prior and likelihood families e.g., conjugate distributions, computing the posterior is typically intractable.
In the next section, we present a variational framework for scalable approximation of the posterior for Bayesian causal discovery.
\section{Causal Discovery via \fullname{}}\label{sec:method}
As discussed previously, we are interested in learning the posterior distribution \(p(W, \Sigma \mid X_1^n)\) over the unknown SEM parameters \(\left\{{{ W}, { \Sigma}}\right\}\) given an observed dataset \(X_1^n\).
Unlike point estimators, such a posterior distribution will allow us to quantify the uncertainty in estimation. Our framework, \fullname{} (\name{}) allows us to tractably estimate this posterior.

As a first step, our approach involves parametrizing the adjacency matrix \(W\) as the product of a permutation matrix \(P\) and a strictly lower-triangular matrix \(L\), so that
  \(W = PLP^\top\).
  In graphical terms, \(L\) is a weight matrix for a canonical DAG with a fixed ordering, while pre- and post-multiplication by \(P\) and \(P^\top\) modifies the ordering of nodes.
\(L\) is parameterised by a vector of weights \(l \in \R^{d(d-1)/2}\), and the constraint that \(P\) is a permutation ensures that \(W\) is the adjacency matrix of a DAG.

Our goal is to obtain the posterior distribution \(p(P, L, \Sigma|X_1^n)\). Due to the intractable partition function, we cannot directly compute the posterior. We turn to variational inference 
to deliver a tractable approximation to the posterior~\citep{jaakkola2000bayesian}.
The key idea here is to cast inference as an optimization problem, wherein we approximate the true posterior with a tractable family of distributions $q_\phi(P, L, \Sigma)$ parameterized by $\phi$ and optimize these parameters $\phi$ to minimize the KL divergence between the approximate and true posterior distributions:
\begin{align}
  &\dkl \left(q_\phi(P, L, \Sigma)\| p(P, L, \Sigma\mid X_1^n)\right) \nonumber  \\
                                    & = -\underbrace{
                                    \Ex{(P, L, \Sigma)\sim q_\phi}
                                    {\log p(X_1^n|P, L, \Sigma) - \log \frac{q_\phi(P, L, \Sigma)}{p(P, L, \Sigma)}}
                                    }_{\text{ELBO}(\phi)} \; + \log p(X_1^n).
\end{align}
Hence, minimizing the KL divergence above corresponds to maximizing the evidence lower bound (ELBO) w.r.t.\ variational parameters $\phi$.
With
a sufficiently expressive variational family from which to choose \(q\), maximizing the ELBO recovers the true posterior as \(q_\phi(P, L, \Sigma) = p(P, L, \Sigma|X_1^n)\).

In practice, we face  important  
modeling choices which have a substantial impact on the quality of the posterior obtained, as well as the difficulty of optimizing the ELBO.
These include the variable ordering to use when factorizing  \(q_\phi(P, L, \Sigma)\) using the chain rule, choice  of variational family for the individual (conditional) factors, as well as the prior distribution \(p(P, L, \Sigma)\).
We discuss these algorithmic design choices next.
\subsection{Factorization of Approximate Posterior}
Approaches to variational inference with multiple sets of latent variables often use a mean-field factorization \citep{jaakkola2000bayesian}, in our case
corresponding to \(q_\phi(P, L, \Sigma) = q_{\phi}(P)q_{\phi}(L)q_{\phi}(\Sigma)\). This mean-field approach can often simplify the
optimization of the ELBO, but severely limits the expressiveness of the approximate posterior. 
For example, consider a two-dimensional linear-Gaussian SEM with non-equal variances (and therefore has non-uniquely identifiable parameters).
Under infinite data, the posterior density concentrates on the two (observationally undistinguishable) maximum-likelihood solutions: an edge \(x_1 \to x_2\) with some weight \(l_1\)and an edge \(x_1 \leftarrow x_2\) with another weight \(l_2\). The posterior concentrates to a bimodal distribution, with density around the region \((l_1, P_1)\) and around \((l_2, P_2)\). A mean-field factored posterior \(q_\phi(L)q_\phi(P)\) cannot represent this correlated density. 
Empirically we observe that such a factored posterior leads to a worse ELBO, illustrated in ablation experiments in Section~\ref{sec:ablation}.

For \name{}, we use a factorization \(q_\phi(P, L, \Sigma) = q_\phi(P|L, \Sigma)q_\phi(L, \Sigma)\), sampling \(L\) and \(\Sigma\) jointly first, then conditionally sampling \(P\) based on these values, using a neural network to learn the parameters of the conditional distribution \(q_\phi(P|L, \Sigma)\). This leads to an ELBO
\begin{align}
\label{eq:factorized_order}
  \mathbb{E}_{(L, \Sigma) \sim q_\phi}\left[\Ex{P \sim q_\phi(\cdot | L, \Sigma)}{\log p(X_1^n|P, L, \Sigma) - \log \frac{q_\phi(P|L, \Sigma)}{p(P|L, \Sigma)}} - \log \frac{q_\phi(L, \Sigma)}{p(L, \Sigma)}\right] 
\end{align}

\subsection{Variational Families}
An important design choice in variational methods is the variational family used. A distribution \(q_\phi\) over latents $z$ used in ELBO optimization must support two operations:
drawing a sample \(z \sim q_\phi\), and computing \(\log q_\phi(z) - \log p(z)\). Depending on the prior distribution \(p\), it may be additionally possible to compute the term
\(\Ex{z \sim q_\phi}{\log q_\phi(z) - \log p(z)} = \dkl\left[q_\phi, p\right]\) in closed form, possibly reducing variance compared to Monte Carlo estimates~\citep{roederStickingLandingSimple2017}.
It is also desirable that sampling \(z \sim q_\phi\) can be written as \(g_\phi(\gamma), \gamma \sim q_0\), i.e.\ that \(z\) is obtained by sampling from a fixed distribution $q_0$ and
transformed through a parameterized, differentiable sampling path $g_\phi$. This lets us use pathwise gradient estimators~\citep{mohamedMonteCarloGradient2019}, which typically have lower variance
than the score-function  alternatives~\citep{williams1992simple}. 

\subsubsection{Distribution over Weights \& Noise Variances}\label{sec:l_var_family}
We consider two different variational families for the distribution over weights and noise variances, \(q_{\phi}(L, \Sigma)\), depending on the modeling assumptions over the noise. 
Under the equal variance modeling assumption, we parameterize our variational family as a (diagonal covariance) normal distribution, with \(\phi\) directly encoding the mean and variance of the \(d(d-1)/2\) random variables for \(L\) and the single random variable for \(\sigma\). 
We use this simple distribution since we expect the posterior over \(L\) to be relatively unimodal in the identifiable equal variance case.

In the non-equal variance case, and in the experiments using real-world data, we use a normalizing flow \citep{rezende2015variational} for \(q_\phi(L, \Sigma)\). We expect that in this case the distribution over \(L\) could be much more complicated, and so a more expressive density model is desirable. We use continuously indexed normalizing flows \citep{cornish2020relaxing}, a recently-developed family of flows offering good performance on multimodal densities.
As desired, both the normal distribution and flows have pathwise gradient estimators. 
\subsubsection{Distribution over permutations}\label{sec:perm_relax}
Since the set of $d$-dimensional permutation matrices \(\mathcal{P}_d\)  is discrete and its size scales combinatorially with $d$, it is challenging to specify a variational family of distributions over $P \in \mathcal{P}_d$ that permits both density estimation and sampling for low-variance stochastic optimization of the ELBO objective in equation~\eqref{eq:factorized_order}.
Since permutations are discrete, pathwise gradient estimators do not exist.
Hence, we  consider relaxations to distributions over permutations.
Our base distribution is the Boltzmann distribution over \(\mathcal{P}_d\), parametrised by \(T \in \R^{d\times d}\) with probability \(P_T(P) \propto \exp \langle T, P \rangle\) for \(P \in \mathcal{P}_d\).

\textbf{Density estimation.} Computing the partition function for the Boltzmann distribution, \(\sum_{P \in \mathcal{P}_d} P_T(P)\) involves an expensive enumeration and is therefore intractable to evaluate in high dimensions.
In order to approximate the partition function, we follow~\citep{li2021discovering} in noting that the partition function is equal to the matrix permanent \(\perm (\exp T)\). This can in turn be approximated tractably via  the Bethe permanent, denoted as \(\perm_B(\exp T)\).
 The Bethe permanent is known to satisfy \(\log \perm T - \frac{d}{2}\log 2 \leq \log \perm_B T \leq \log \perm T\), so that the density will be over-estimated, by no more than a factor of \(\frac{d}{2}\log 2\)~\citep{anari2019tight}. 
We refer the reader to appendix C in~\cite{mena2020sinkhorn} for an efficient implementation of the Bethe permanent estimator based on message passing.

\textbf{Pathwise Gradient Estimation.} Exact sampling from the Boltzmann distribution is challenging for similar reasons as exact density estimation. 
Moreover, even tractable low-rank approximations to the Boltzmann distribution based on Gumbel-Matching distributions~\citep{tomczak2016some} are not useful as they involve non-differentiable operations and so cannot be used to derive a pathwise gradient estimator.
Instead, we use a relaxation to the Gumbel-Matching distribution, the Gumbel-Sinkhorn distribution~\citep{menaLearningLatentPermutations2018}.

To draw a sample from the Gumbel-Sinkhorn distribution with parameters $T$, we calculate \(S((T + \gamma)/\tau)\), with \(S\) the Sinkhorn operator~\citep{sinkhorn1964relationship}, \(\gamma\) a matrix of i.i.d standard Gumbel noise and \(\tau\) a temperature hyperparameter. \(S(T)\) returns the fixed point obtained from  repeated row and column normalization, starting from the elementwise \(\exp\) of \(T\). In the limit of an infinite number of iterations, this returns a doubly stochastic matrix. As \(\tau\) approaches zero, the samples approach samples from \(\mathcal{P}_d\), with a distribution given by the Gumbel-Matching distribution. A proof of this fact is given in the appendix of~\cite{menaLearningLatentPermutations2018}. As the Sinkhorn algorithm is a differentiable function of standard Gumbel noise, we can use a pathwise gradient estimator of gradients involving samples from the Gumbel-Sinkhorn distribution. Additional implementation details for our Sinkhorn approach are in the appendix, Section~\ref{sec:addit-exper-deta}.
\subsection{Prior Distributions}
A key aspect of any probabilistic model is the choice of prior distribution for the unknown parameters. The prior incorporates domain knowledge into the problem. 
Moreover, specific choices of prior can be computationally friendly.

\textbf{Gaussian Prior.}
We show in the appendix (Section~\ref{sec:marginalization-proof}) that if we choose the prior over edge weights to be an isotropic Gaussian, we can analytically marginalize out the weights, only requiring the distribution over \(P\) to characterise the full posterior.
Once we have \(P\), it is straightforward to obtain \(L\), since it is a regression problem which can be solved tractably~\citep{tibshirani1996regression}. Although it is very convenient to avoid modeling a distribution over \(L\),
in practice we find that the Gaussian generative assumption on the weights is not particularly useful for datasets which we would like to analyse, for which the underlying DAGs are typically
sparse. A sample from the distribution over DAGs with Gaussian edge weights will likely have many large- or moderate-weight edges. 

\textbf{Laplace Prior.} Previous work finding the maximum-likelihood solution of equation~\eqref{eq:3} has added a term \(\lambda \|W\|_1\)
penalizing the \(L_1\) norm of the adjacency matrix \citep{ngRoleSparsityDAG2020, zhengDAGsNOTEARS2018}, which can be interpreted as imposing an isotropic Laplace prior. The Laplace prior is known to induce sparsity in the posterior \citep{tibshirani1996regression}, but there is generally no way to choose the regularization coefficient \(\lambda\) without cross-validation. 

\textbf{Horseshoe Prior.} Given the limitations of the above choices of priors, we instead propose to use a horseshoe prior on \(L\). 
The horseshoe prior has a sharp peak at zero and relatively flat tails which tend to induce sparsity in the posterior while not significantly penalizing larger coefficients \citep{carvalho2009handling}. Mathematically, a variable \(\beta_i\) has a horseshoe distribution if it is the result of first drawing a random variable \(\lambda_i \sim C^+(0, 1)\) from a half-Cauchy distribution, then sampling \(\beta_i \sim \mathcal{N}(0, \lambda_i^2\eta^2)\). 
The parameter \(\eta\) can be adjusted to encode the prior belief on the degree of the DAG generating the data. 
Based roughly on~\cite{piironen2017hyperprior}, we suggest a rule of thumb of setting \(\eta \approx \rho/(d\sqrt{n})\),
with \(\rho\) the prior belief of the average degree of the DAG. This results in a sparsity prior that penalises more stringently with more data, similarly to the BIC score~\cite{schwarz1978estimating}. 
\subsection{Overall Approach}
Incorporating all the above design choices, we obtain our overall algorithm for \name{}. The pseudocode is shown in Algorithm~\ref{alg:liga-vi}. 
Here, \(q_{\phi}(L, \Sigma)\) is parameterized as either a  normal distribution or a normalizing flow depending on the modeling assumption of equal or non-equal variances.
The distribution \(q_{\phi}(P|L, \Sigma)\) is a Gumbel-Sinkhorn relaxation of the Gumbel-Matching distribution over permutations, parameterized by a function \(h_\phi\) conditioned on \(L, \Sigma\). For the function \(h_\phi\), we use a simple two-layer multi-layer perceptron. 
For stochastic optimization w.r.t.\ this distribution, we additionally find it useful to use the straight-through gradient estimator \citep{bengio2013estimating}. This means that on the forward pass of the backpropogation algorithm, we obtain the \(\tau \to 0\) limiting value of the Sinkhorn relaxation using the Hungarian algorithm \citep{kuhn1955hungarian}, giving a hard permutation $P$. On the backward pass the gradients are taken with respect to the finite-\(\tau\) doubly-stochastic matrix ${\tilde P}$.
The prior is given by \(p(P, L, \Sigma) = p(P)p(L)p(\Sigma)\) where \(p(L)\) is a horseshoe prior, \(p(P)\) is a uniform prior over permutations and \(p(\Sigma)\) is
a relatively uninformative Gaussian prior on \(\log \Sigma\).

\begin{algorithm}[th]
  \SetAlgoLined{}
  \DontPrintSemicolon{}
    \SetKwInOut{Input}{Input}\SetKwInOut{Output}{output}
    \Input{data \(X_1^n\), Gradient-based optimizer \texttt{step}, temperature hyperparameter $\tau$ \;}
    Initialize parameterized distribution \(q_{\phi}\),
    neural network \(h_{\phi}(L, \Sigma)\)\;
    \While{not converged}{
      Draw \(L, \Sigma \sim q_{\phi}(L, \Sigma)\)\;
      Compute logits \(T = h_\phi(L, \Sigma)\) \;
      Draw \(\gamma \in \R^{d \times d}\) i.i.d from standard Gumbel\;
      Compute soft \({\tilde P} = S((T + \gamma)/\tau)\), hard \(P = \operatorname{Hungarian}({\tilde P})\) \;
      Compute \(g = \nabla_{\phi}\left[\operatorname{ELBO}(\phi)\right]\) from equation~\eqref{eq:factorized_order} with sampled \(P, L, \Sigma\), using \(P\) in the forward pass and \({\tilde P}\) in the backward pass\;
      Update $\phi$ via \texttt{step} using gradient \(g\)\;
    }
    \caption{\fullname{} (\name{}) }\label{alg:liga-vi}
  \end{algorithm}   
\section{Related Work}\label{sec:related-work}
\textbf{Structure Learning for Bayesian Networks}: The field of Bayesian structure learning investigates how to infer the structure of Bayesian networks from data.
Learning the structure of graphical models from data is known to be NP-hard \citep{chickering1996learning}. Nevertheless, two main families of approaches have been developed to tackle this problem. 
The first are constraint-based approaches \citep{spirtesCausationPredictionSearch1993}, aiming to infer conditional independences from the data and learn a structure consistent with these independences. 

The second family of approaches are score-based methods \citep{malone2017asl, koller2009probabilistic}, which assign a score to a proposed DAG based on the goodness-of-fit to the data, and search over the space of DAGs to find a structure which maximizes the score. The combinatorial size and discrete nature of the set of possible DAGs in \(d\) dimensions present the main difficulties to this approach. Several competing scores have been proposed \citep{schwarz1978estimating, akaike1998information}, which generally have an either implicit or explicit term penalizing complexity to avoid overfitting to the data.
Recent work has focussed on developing approximate algorithms that can return a high-scoring graph with probabilistic guarantees \citep{nie2016learning} and on exact methods for graphs under structural assumptions, such as an upper bound on the number of parents of a node \citep{cussensBayesianNetworkStructure2017, scanagatta2015learning}, similar to the classic Chow-Liu algorithm \citep{chow1968approximating}. Using these techniques it is possible to search over all graphs where nodes have at most two parents for dimension \(d > 1000\) \citep{sheehan2014improved}. 

Another set of works improve sampling from distributions over DAGs, e.g.\ approaches sampling over orderings instead of over DAGs~\citep{teyssierOrderingbasedSearchSimple2005}.
Given a particular network structure, there are several approaches to learning DAG edge weights, even in the presence of latent variables~\citep{merkle2018blavaan}.

\textbf{Continuous Relaxations for Structure Learning}: The problem of learning the structure of probabilistic graphical models on \emph{undirected} graphs can be formulated as a convex optimization problem \citep{friedman2008sparse}, which allows fast inference even for large graphs. However, the acyclicity constraint for learning directed graphs means continuous optimization cannot be applied directly. 
In recent years, relaxation-based approaches to structure learning have emerged, based on the idea of learning the adjacency matrix corresponding to a DAG via optimization over a matrix \(W \in \R^{d\times d}\), instead of in the set of adjacency matrices corresponding to DAGs. This approach was introduced in \citet{zhengDAGsNOTEARS2018}, with the least-squares loss \(\|X - W^\top X\|_F^2\) used to measure the fit of the data to \(W\). To encourage \(W\) to approximately form an adjacency matrix, a penalty \(h(W)\) was added to the loss, where \(h(W) = 0\) only if \(W\) is a DAG.\ Originally the penalty \(h(W) = \operatorname{Tr}e^{W \odot W}\) was used, and subsequently different penalties have been proposed \citep{yuDAGGNNDAGStructure2019}.

Particularly of interest to us is GOLEM~\citep{ngRoleSparsityDAG2020}, which specifically investigates relaxed optimization approaches for the linear-Gaussian SEM that we consider. They point out the similarity of the squared loss to the log-likelihood in equation~\eqref{eq:3}, and show that the additional term \(\log \det (I - W)\) also serves as a penalty for non-DAG \(W\). This allows optimization over \(W\) to maximize the likelihood without the augmented Lagrangian optimization procedure used in \citet{zhengDAGsNOTEARS2018}.

\textbf{Learned Orderings}: There are strong connections between our SEM setting and the problem of learning an ordering for autoregressive models. An SEM is similar to an autoregressive model~\citep{akaike1969fitting} in the sense that the value at the \(i\)th index depends on (possibly) all of the values before \(i\). It has been observed that certain orderings for autoregressive models can lead to higher likelihood~\citep{germain2015made}, and subsequent work~\citep{jainLocallyMaskedConvolution2020, li2021discovering} has developed methods to find the orderings which lead to the best performance from the autoregressive model. The closest to our work is~\citet{li2021discovering}, which uses latent permutations to learn an ordering for an autoregressive model. However, this approach doesn't allow reparameterised gradients due to the hard assignment for the autoregressive order, and the REINFORCE~\citep{williams1992simple} gradient estimator is used instead of lower-variance reparameterised gradients.% \citep{mohamedMonteCarloGradient2019}. 

\section{Experiments}\label{sec:experiments}
In this section we study the empirical performance of our method.
On synthetic data, we show that our distributional approach outperforms baselines, including the maximum likelihood approach, in the low-data regime in terms of edges identified correctly.
On a real-word protein dataset \citep{sachsCausalProteinsignalingNetworks2005}, we also identify more edges correctly compared to the maximum-likelihood method.
On a toy causal inference problem, we outperform competing methods. 
Finally, we carry out an ablation with various parts of our algorithm, showing the degradation in performance if we remove key parts. In all cases we use the structural Hamming distance (SHD)~\citep{tsamardinos2006max} to quantify how close a certain DAG is to another. The SHD measures how many insertions, deletions or flips are required to turn one graph into another. We also use the CPDAG SHD (SHD-C)~\citep{tsamardinos2006max} which computes SHD up to the equivalence class of completed partially directed acyclic graphs (CPDAG)s~\citep{richardson2002ancestral}. We give additional experimental details in the appendix~\ref{sec:addit-exper-deta}. In every case when we give the SHD produced by our model, this represents the SHD marginalized over our posterior, obtained by taking 100 samples of \(W\) from our converged posterior, computing the SHD and taking the average of these SHDs. Error bars are given by 5 random seeds for the data generation or choice of data subset. 
\begin{figure*}
  \centering
  \includegraphics[width=0.92\textwidth]{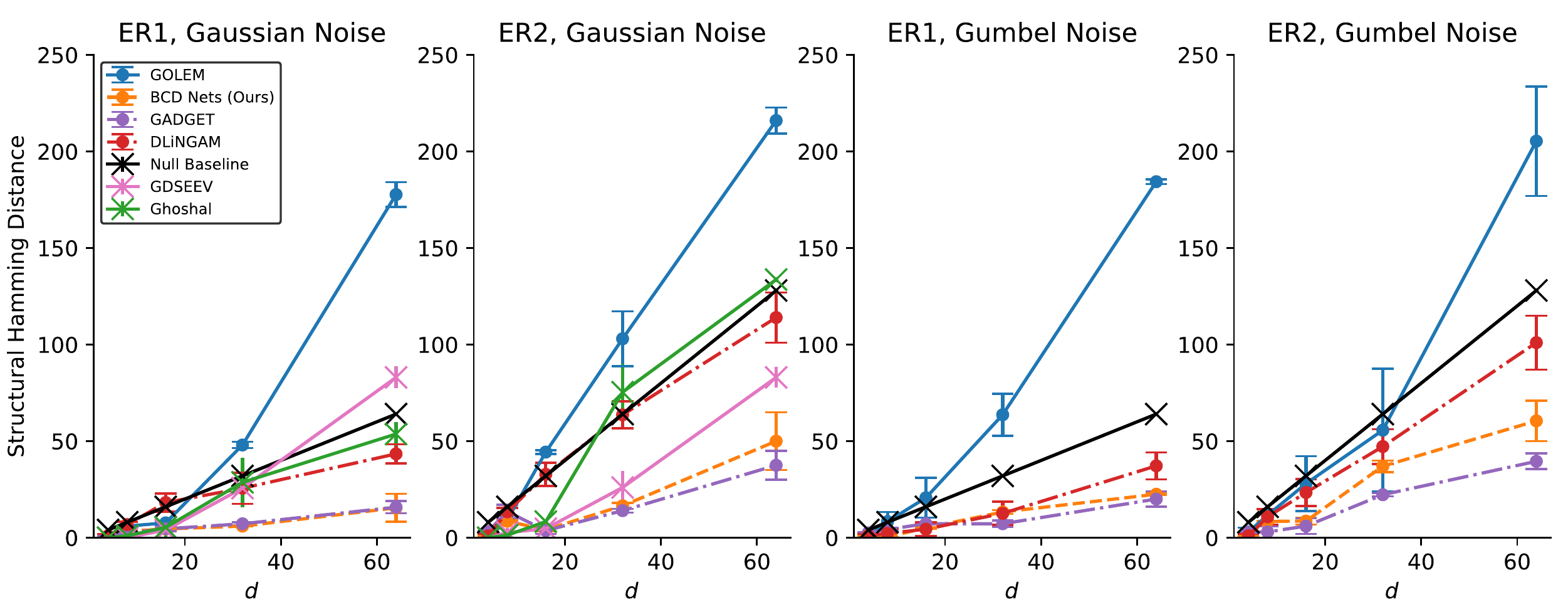}
  \caption{Distance between true graph and that estimated from \name{}, our variational approach, compared to several baselines described in the main text. \textit{Lower is better.}
}\label{fig:synthetic-ev}
\end{figure*}
\vspace{-0.3cm}
\subsection{Synthetic Data}\label{sec:comp-maxim-likel}
We study our model's performance in the low-data regime, with \(n=100\) data points. We expect that when there is little data, the maximum-likelihood DAG may not be particularly representative of the DAGs that could have plausibly made the data, and so a Bayesian approach which can assign posterior density to multiple candidates may correctly identify more edges.
In this experiment, we draw Erdős-Renyi graphs~\citep{erdHos1960evolution} with average degree equal to 1 or 2, denoted ER1 and ER2. The weights are generated uniformly in \(\left\{-2, -0.5\right\}\cup \left\{0.5, 2\right\}\) following~\cite{ngRoleSparsityDAG2020}.
The data is generated in the equal-variance case so that the ground-truth is identifiable and the SHD is meaningful. We evaluate with Gaussian noise as well as with Gumbel noise to test how our method behaves under a misspecified likelihood. The distribution over the weights also induces some misspecification, since the weights are not drawn from the modelled prior of a horseshoe.

We compare against GOLEM~\citep{ngRoleSparsityDAG2020}, which attempts to find the maximum-likelihood solution of equation~\eqref{eq:3} using the continuous relaxation introduced in~\citep{zhengDAGsNOTEARS2018}. We found that the hyperparameters given by the authors did not lead to optimal performance with low data, so we found the best ones via a cross-validation scheme given in appendix \ref{sec:addit-exper-deta}. We also include results from GADGET \citep{viinikka2020towards},  a recent highly-optimized state-of-the-art Monte Carlo method for Bayesian causal discovery for DAGs, using 300,000 iterations with 48 parallel chains. To compare against a baseline that is widely used in practice, we also evaluate the DirectLiNGAM method~\cite{shimizu2011directlingam}, designed for non-Gaussian likelihoods. Also included are two methods designed specifically for Gaussian DAGs, the greedy directed acyclic graph search with equal error variance (GDSEEV) method \cite{peters2014identifiability}, and the main method (`Ghoshal') from \cite{ghoshalLearningLinearStructural2018}. For the latter, we were unable to find an implementation by the authors and used our own implementation. 
We use the equal-variance formulations of GOLEM and \name{}.

Results are shown in figure~\ref{fig:synthetic-ev}. The null baseline is a method that simply predicts no edges.
We observe that we reliably obtain significantly lower SHD than MLE approaches, especially at higher \(d\) and degree where we tend to recover a significant fraction of the edges.
For example, in the Gaussian noise case with average degree 2, for dimension 64 we get an SHD of \(59 \pm 14\) compared to GOLEM's \(205 \pm 30\).
This corresponds to double the true positive rate while having a false discovery rate one-third of GOLEM's.
We give full plots of the true positive rate, false positive rate and false discovery rates in the appendix, Section~\ref{sec:addit-pred-stat}.
The Bayesian GADGET method also obtains better results than MLE methods, with BCD Nets' match to this exact sampling method suggesting that our approach finds a variational posterior close to the real posterior.
Compared to GADGET, our method is generally faster (see table \ref{tab:3}) and gives an explicit parametric form for the posterior, instead of simply giving samples. Furthermore, the ELBO is a natural metric to evaluate convergence, while for GADGET it is unclear how many iterations are required to mix\footnote{An earlier version of this paper reported worse SHDs for GADGET after using 50,000 steps. Increasing the number of iterations by 6 times resulted in much better mixing}. DirectLiNGAM is competitive on the Gumbel likelihood, fitting its linear non-Gaussian framework.
\subsection{Protein Dataset}
We also evaluate on a benchmark protein signalling dataset~\cite{sachsCausalProteinsignalingNetworks2005}. The \(d=11\)-dimensional dataset consists of \(n=853\) observations, with an expert-provided ground-truth graph. The true structure has 17 edges. We train on random draws of 100 observations.
The results are shown in table~\ref{tab:sachs}. For GOLEM, both the equal-variance and non-equal variance maximum-likelihood methods perform worse than predicting no edges at all. Our equal variance method does not perform much better.
Our non-equal variance method performs much better than maximum likelihood, obtaining an SHD of 15 from only \(100\) samples, close to the SHD of \(14\) obtained by GOLEM-NV when using the entire \(853\) data points (reported in \citet{ngRoleSparsityDAG2020}).
NOTEARS, from \cite{zhengDAGsNOTEARS2018} performs better than GOLEM. Meanwhile, the score-based GES method also achieves an SHD-C of 14, showing that our method approaches the performance of methods which enumerate all graphs (but scale badly). The non-probabilistic DLiNGAM method achieves similar SHD, although has the lowest SHD-C, indicating it gets some edges correct but the direction wrong. Finally, GADGET~\citep{viinikka2020towards} achieves a slightly lower SHD of 13.9, within error of our approach and GES.
\begin{wraptable}{R}{0.5\linewidth}
  \vspace{-0.65cm}
  \centering
\caption{Causal discovery approaches
  on the protein dataset with reduced data (\(n=100\))}\label{tab:sachs}
{\small
\begin{tabular}{llll}
           & \# Edges & SHD & SHD-C \\ \toprule
GOLEM-EV   &   1.5 \(\pm\) 1.3  & 18.5 \(\pm\) 1.3  &  18.5 \(\pm\) 1.3     \\
GOLEM-NV   &   1.5 \(\pm\) 1.3  & 18.5 \(\pm\) 1.3  &  18.5 \(\pm\) 1.3     \\
  NOTEARS    &   18.5 \(\pm\) 0.8 &  16.5 \(\pm\) 0.9 & 17 \(\pm\) 1.0\\
  GES  & 12 \(\pm 0.9\) & --- &  14.6 \(\pm\) 2.0\\
           PC & 4.6 \(\pm\) 0.5 & --- & 14.0 \(\pm\) 0.6\\
  GADGET & 4.7 \(\pm\) 0.5 & 13.9 \(\pm\) 1.2 & 13.7 \(\pm\) 0.6\\
  DLiNGAM & 4.6 \(\pm\) 0.5 & 14.8 \(\pm\) 1.0 & 12.4\(\pm\) 0.5\\
\name{}-EV &   11.3 \(\pm\) 1.2  & 19.5 \(\pm\) 0.3    & 19.4 \(\pm\) 0.1      \\
\name{}-NV &   9.2 \(\pm\) 2.0    & 14.7 \(\pm\) 0.9   & 14.0 \(\pm\) 1.0
\end{tabular}
}
\vspace{-0.25cm}
\end{wraptable}
\subsection{Causal Inference}\label{sec:causal-inference}
To illustrate how our distributional approach could be used for causal inference, we test the ability of our method to make interventional predictions. In this experiment we generate a synthetic ER graph of degree 1 as in Section \ref{sec:comp-maxim-likel}. We then choose a random edge in the graph between nodes \(i, j\), with \(x_i \to x_j\). Using the ground-truth parameters \(W^*, \Sigma^*\), we choose a random value \(a\) and sample \(x_j\sim \operatorname{do}(x_i=a|W^*, \Sigma^*)\), by directly intervening in the data-generating process. For an estimated \({\hat W}, {\hat \Sigma}\) we can also sample \(x_j\sim \operatorname{do}(x_i=a|{\hat W}, {\hat \Sigma})\). For our distributional approach, we marginalize over the final posterior distribution of parameters \(q_\phi(W, \Sigma)\), drawing \(x_j\) from the distribution with probability \(\Ex{W, \Sigma \sim q_\phi}{P(\operatorname{do}(x_i=a|W, \Sigma))}\). We then compare the sampled empirical distribution of \(x_j\) to the ground truth interventional distribution. Our marginalization over the posterior allows us to get significantly closer to the true interventional distribution, measured by the Wasserstein distance (e.g. 0.25 for the linear-Gaussian graph at \(d=64\), compared to 2.8 for GOLEM). Full results as a function of \(d\) are given in appendix ~\ref{sec:caus-infer-exper}.

\subsection{Model Running Time}
We expect the time taken to train our model will asymptotically scale as \(\mathcal{O}(d^3)\), similarly to maximum-likelihood methods~\cite{ngRoleSparsityDAG2020, zhengDAGsNOTEARS2018}. In table~\ref{tab:3} we give the time taken to train to convergence as we vary the dimension \(d\). The time taken for GOLEM includes the time required to choose the sparsity parameter \(\lambda\) via cross-validation. \name{} didn't need any cross-validation to choose sparsity parameters. All the methods were run on the same hardware, a single Nvidia 2080Ti GPU with 16 CPUs.
Our method takes more time to converge than GOLEM. This is not surprising, since we are training a neural network with many Sinkhorn iterations per optimization step. The Ghoshal algorithm~\cite{ghoshalLearningLinearStructural2018} is very fast in comparison, consisting only of a precision estimation step then \(d\) stages of matrix multiplication. At high \(d\), GADGET is slower than our method. We found that speeding up GADGET by reducing the number of iterations resulted in dramatically reduced performance.
\begin{wraptable}{R}{0.5\linewidth}
  \vspace{-0.65cm}
  \centering
  \caption{The computing time required, in minutes, to converge to a solution for several approaches.}\label{tab:3}
  \vspace{0.1cm}
\begin{tabular}{@{}llllll@{}}
$d$            & 8   & 16   & 32   & 64   & 128  \\ \midrule
  GOLEM        & 25 & 30  & 40  & 65  & 90  \\
  GADGET       & 40  & 150  & 385 & 635  & 1200   \\
  GHOSHAL      & \textless 1 & \textless 1 & \textless 1  & 3  & 15  \\
  \name{} (Ours) & 50 & 160 & 300 & 350 & 900 \\ \bottomrule

\end{tabular}
\vspace{-0.3cm}
\end{wraptable}

\subsection{Increasing Dataset Size}
We expect that our method performs best in the low-data regime, where a Bayesian approach with correctly-specified priors has an advantage over non-Bayesian methods. To show this, we perform experiments with increased amounts of data, results of which are shown in figure \ref{fig:high-n}. We see that the advantage of \name{} is reduced as the quantity of data used is increased.
We suspect that our variational method suffers from a more challenging optimization problem as the posterior becomes more peaked around the correct solution. Methods such as amortized inference or sequential Monte-Carlo may help to combat this problem and improve our model at larger \(n\). 

\begin{figure}
  \centering
  \includegraphics[width=0.9\textwidth]{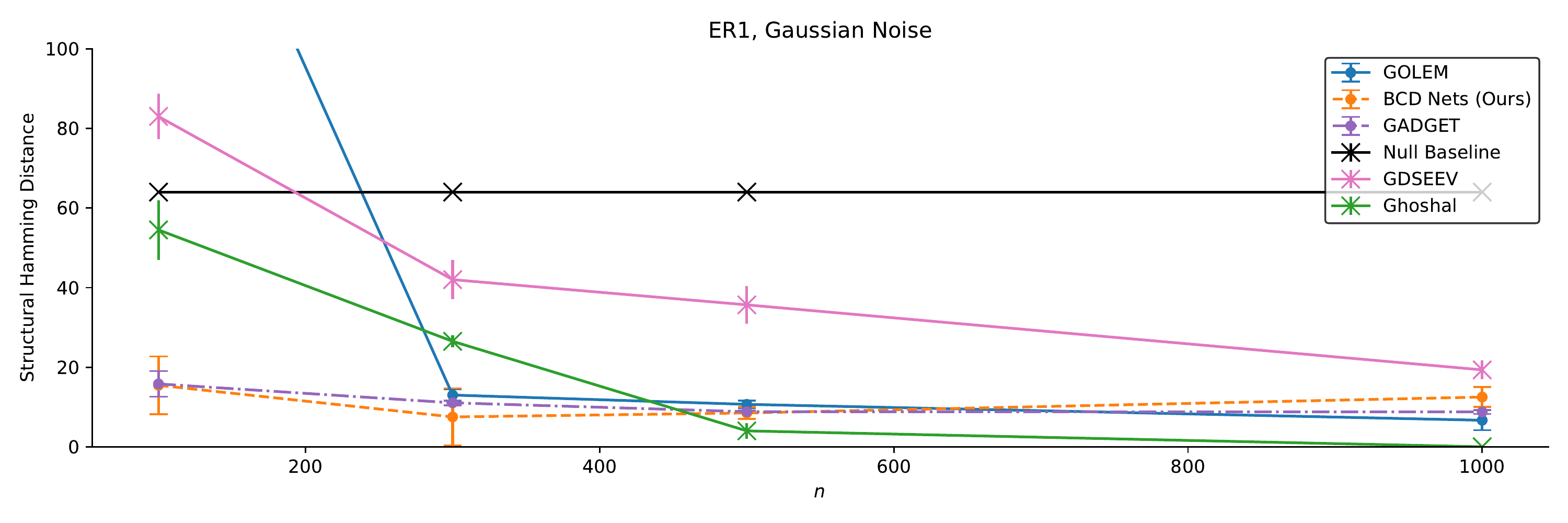}
  \caption{Structural Hamming Distance to the ground truth as a function of \(n\), on 64-dimensional, degree 1 Erdős-Renyi graphs. With large amounts of data, baselines catch up to \name{}.}\label{fig:high-n} 
\end{figure}

\subsection{Ablation}\label{sec:ablation}
Finally we perform an ablation of our algorithm to illustrate the importance of architectural choices. We use the \(d=32\) case, average degree 1, with Gaussian likelihood. We report the change in performance with a factored posterior (\textbf{Mean-Field}) and with a Laplace prior on the weights (\textbf{Laplace}) instead of a Horseshoe prior. We also report performance with a fixed \(100\) Sinkhorn steps (\textbf{Sinkhorn-100}), as opposed to the adaptive number that we use in our approach. We also report the KL-divergence between the sample covariance and the covariance induced by the sampled parameters \((\Sigma, L, P)\), which incorporates how well the posterior approximates the distribution of \(L\), unlike SHD\@. We give the results in table~\ref{tab:1} and \ref{tab:2}, in the appendix. Changing any of the parts of the algorithm result in SHD increasing from 11 to around 30. Interestingly, under the Laplace prior the SHD is high but the KL-divergence is only modestly higher than with the Horseshoe.
This indicates a learned DAG which can generate the data, but with extra edges due to an ineffective sparsity prior.

\section{Conclusion}\label{sec:conclusion}
We introduce \fullname{}, a framework for performing Bayesian causal discovery for linear-Gaussian structural equation models. \name{} exploit recent advances in variational inference to flexibly model the posterior distribution over SEM parameters given data, outperforming methods that only return point estimates. An explicit posterior is also useful in the high-stakes, low-data regimes where causal inference is increasingly used \citep{varian2016causal, shalitEstimatingIndividualTreatment2016}.

On the other hand, while indications are that our method may be robust to a small amount of misspecification, the validity of the linear modelling assumption should be carefully considered in applications. Assuming a linear relationship is present where a highly non-linear one exists could lead to harmfully incorrect inferences, particularly for minority groups with complex and under-studied variable interactions.
Furthermore, our approach assumes that no \emph{unobserved confounders} are present: variables which influence the observed variables but are not themselves observed. Since it is difficult to obtain useful inferential results under the effects of possibly arbitrary unknown confounders, the assumption of no unobserved confounders is standard in many areas of causal inference~\citep{peters2014identifiability}. However, this assumption is unlikely to hold exactly in real applications, and so the possible influence of unobserved confounders must be seriously considered and their influence reduced as much as possible before using our method. 
Future work could explore 
improvements to inference speed by replacing the Sinkhorn with more efficient algorithms from optimal transport.% \citep{altschulerNearlinearTimeApproximation2017}.
{\small \section*{Acknowledgements}}
This research was supported by NSF(\#1651565, \#1522054, \#1733686), ONR (N000141912145), AFOSR (FA95501910024), ARO (W911NF-21-1-0125) and the Sloan Fellowship. We thank the anonymous reviewers for their helpful comments.

{\small \section*{Funding Transparency Statement}}

1. \textbf{Funding}: 
This research was supported by NSF(\#1651565, \#1522054, \#1733686), ONR (N000141912145), AFOSR (FA95501910024), ARO (W911NF-21-1-0125) and the Sloan Fellowship

2. \textbf{Competing Interests}: 
There are no competing financial interests that could be perceived to influence the contents of this work.
\newpage
\bibliography{master}

\begin{thebibliography}{63}
\providecommand{\natexlab}[1]{#1}
\providecommand{\url}[1]{\texttt{#1}}
\expandafter\ifx\csname urlstyle\endcsname\relax
  \providecommand{\doi}[1]{doi: #1}\else
  \providecommand{\doi}{doi: \begingroup \urlstyle{rm}\Url}\fi

\bibitem[Akaike(1998)]{akaike1998information}
Hirotogu Akaike.
\newblock Information theory and an extension of the maximum likelihood
  principle.
\newblock In \emph{Selected Papers of Hirotugu Akaike}, pages 199--213.
  {Springer}, 1998.

\bibitem[Akaike(1969)]{akaike1969fitting}
Hirotugu Akaike.
\newblock Fitting autoregressive models for prediction.
\newblock \emph{Annals of the institute of Statistical Mathematics},
  21\penalty0 (1):\penalty0 243--247, 1969.

\bibitem[Anari and Rezaei(2019)]{anari2019tight}
Nima Anari and Alireza Rezaei.
\newblock A tight analysis of bethe approximation for permanent.
\newblock In \emph{FOCS}, 2019.

\bibitem[Bengio et~al.(2013)Bengio, L{\'e}onard, and
  Courville]{bengio2013estimating}
Yoshua Bengio, Nicholas L{\'e}onard, and Aaron Courville.
\newblock Estimating or propagating gradients through stochastic neurons for
  conditional computation.
\newblock \emph{arXiv}, 2013.

\bibitem[Birdal and Simsekli(2019)]{birdal2019probabilistic}
Tolga Birdal and Umut Simsekli.
\newblock Probabilistic permutation synchronization using the riemannian
  structure of the birkhoff polytope.
\newblock In \emph{CVPR}, pages 11105--11116, 2019.

\bibitem[Blei et~al.(2017)Blei, Kucukelbir, and McAuliffe]{blei2017variational}
David~M Blei, Alp Kucukelbir, and Jon~D McAuliffe.
\newblock Variational inference: A review for statisticians.
\newblock \emph{JASA}, 2017.

\bibitem[Carvalho et~al.(2009)Carvalho, Polson, and
  Scott]{carvalho2009handling}
Carlos~M Carvalho, Nicholas~G Polson, and James~G Scott.
\newblock Handling sparsity via the horseshoe.
\newblock In \emph{AISTATS}, 2009.

\bibitem[Chickering(1996)]{chickering1996learning}
David~Maxwell Chickering.
\newblock Learning bayesian networks is {{NP}}-complete.
\newblock In \emph{Learning from Data}. Springer, 1996.

\bibitem[Chow and Liu(1968)]{chow1968approximating}
C~Chow and Cong Liu.
\newblock Approximating discrete probability distributions with dependence
  trees.
\newblock \emph{IEEE transactions on Information Theory}, 14\penalty0
  (3):\penalty0 462--467, 1968.

\bibitem[Cornish et~al.(2020)Cornish, Caterini, Deligiannidis, and
  Doucet]{cornish2020relaxing}
Rob Cornish, Anthony Caterini, George Deligiannidis, and Arnaud Doucet.
\newblock Relaxing bijectivity constraints with continuously indexed
  normalising flows.
\newblock In \emph{ICML}, 2020.

\bibitem[Cussens(2012)]{cussens2012bayesian}
James Cussens.
\newblock Bayesian network learning with cutting planes.
\newblock \emph{arXiv preprint arXiv:1202.3713}, 2012.

\bibitem[Cussens et~al.(2017)Cussens, J{\"a}rvisalo, Korhonen, and
  Bartlett]{cussensBayesianNetworkStructure2017}
James Cussens, Matti J{\"a}rvisalo, Janne~H. Korhonen, and Mark Bartlett.
\newblock Bayesian {{Network Structure Learning}} with {{Integer Programming}}:
  {{Polytopes}}, {{Facets}} and {{Complexity}}.
\newblock \emph{JAIR}, January 2017.

\bibitem[Eigenmann et~al.(2017)Eigenmann, Nandy, and
  Maathuis]{eigenmann2017structure}
Marco~F Eigenmann, Preetam Nandy, and Marloes~H Maathuis.
\newblock Structure learning of linear {{Gaussian}} structural equation models
  with weak edges.
\newblock \emph{arXiv preprint arXiv:1707.07560}, 2017.

\bibitem[Erd{\H{o}}s and R{\'e}nyi(1960)]{erdHos1960evolution}
Paul Erd{\H{o}}s and Alfr{\'e}d R{\'e}nyi.
\newblock On the evolution of random graphs.
\newblock \emph{Publ. Math. Inst. Hung. Acad. Sci}, 5\penalty0 (1):\penalty0
  17--60, 1960.

\bibitem[Fisher(1958)]{fisher1958lung}
Ronald~A Fisher.
\newblock Lung cancer and cigarettes?
\newblock \emph{Nature}, 182\penalty0 (4628):\penalty0 108--108, 1958.

\bibitem[Fragoso et~al.(2018)Fragoso, Bertoli, and
  Louzada]{fragoso2018bayesian}
Tiago~M Fragoso, Wesley Bertoli, and Francisco Louzada.
\newblock Bayesian model averaging: A systematic review and conceptual
  classification.
\newblock \emph{International Statistical Review}, 2018.

\bibitem[Friedman et~al.(2008)Friedman, Hastie, and
  Tibshirani]{friedman2008sparse}
Jerome Friedman, Trevor Hastie, and Robert Tibshirani.
\newblock Sparse inverse covariance estimation with the graphical lasso.
\newblock \emph{Biostatistics}, 9\penalty0 (3):\penalty0 432--441, 2008.

\bibitem[Friedman and Koller(2003)]{friedman2003being}
Nir Friedman and Daphne Koller.
\newblock Being {{Bayesian}} about network structure. {{A Bayesian}} approach
  to structure discovery in {{Bayesian}} networks.
\newblock \emph{Machine learning}, 2003.

\bibitem[Gelman et~al.(2013)Gelman, Carlin, Stern, Dunson, Vehtari, and
  Rubin]{gelmanBayesianDataAnalysis2013}
Andrew Gelman, John~B. Carlin, Hal~S. Stern, David~B. Dunson, Aki Vehtari, and
  Donald~B. Rubin.
\newblock \emph{Bayesian {{Data Analysis}}, {{Third Edition}}}.
\newblock {CRC Press}, 2013.

\bibitem[Germain et~al.(2015)Germain, Gregor, Murray, and
  Larochelle]{germain2015made}
Mathieu Germain, Karol Gregor, Iain Murray, and Hugo Larochelle.
\newblock Made: {{Masked}} autoencoder for distribution estimation.
\newblock In \emph{ICML}, 2015.

\bibitem[Ghoshal and Honorio(2018)]{ghoshalLearningLinearStructural2018}
Asish Ghoshal and Jean Honorio.
\newblock Learning linear structural equation models in polynomial time and
  sample complexity.
\newblock In \emph{International {{Conference}} on {{Artificial Intelligence}}
  and {{Statistics}}}, pages 1466--1475. {PMLR}, March 2018.

\bibitem[Han et~al.(2017)Han, Zhang, Homayouni, and Karmaus]{han2017efficient}
Shengtong Han, Hongmei Zhang, Ramin Homayouni, and Wilfried Karmaus.
\newblock An efficient bayesian approach for gaussian bayesian network
  structure learning.
\newblock \emph{Communications in Statistics-Simulation and Computation}, 2017.

\bibitem[Heckerman et~al.(1999)Heckerman, Meek, and
  Cooper]{heckerman1999bayesian}
David Heckerman, Christopher Meek, and Gregory Cooper.
\newblock A bayesian approach to causal discovery.
\newblock \emph{Computation, causation, and discovery}, 19:\penalty0 141--166,
  1999.

\bibitem[Jaakkola and Jordan(2000)]{jaakkola2000bayesian}
Tommi~S Jaakkola and Michael~I Jordan.
\newblock Bayesian parameter estimation via variational methods.
\newblock \emph{Statistics and Computing}, 10\penalty0 (1):\penalty0 25--37,
  2000.

\bibitem[Jain et~al.(2020)Jain, Abbeel, and
  Pathak]{jainLocallyMaskedConvolution2020}
Ajay Jain, Pieter Abbeel, and Deepak Pathak.
\newblock Locally {{Masked Convolution}} for {{Autoregressive Models}}.
\newblock In \emph{UAI}, 2020.

\bibitem[Koller and Friedman(2009)]{koller2009probabilistic}
Daphne Koller and Nir Friedman.
\newblock \emph{Probabilistic Graphical Models: Principles and Techniques}.
\newblock {MIT press}, 2009.

\bibitem[Kuhn(1955)]{kuhn1955hungarian}
Harold~W Kuhn.
\newblock The {{Hungarian}} method for the assignment problem.
\newblock \emph{Naval research logistics quarterly}, 2\penalty0 (1-2):\penalty0
  83--97, 1955.

\bibitem[Li et~al.(2021)Li, Trabucco, Park, Gao, Luo, Shen, and
  Darrell]{li2021discovering}
Xuanlin Li, Brandon Trabucco, Dong~Huk Park, Yang Gao, Michael Luo, Sheng Shen,
  and Trevor Darrell.
\newblock Discovering autoregressive orderings with variational inference.
\newblock In \emph{ICLR}, 2021.

\bibitem[Little and Rubin(2000)]{little2000causal}
Roderick~J Little and Donald~B Rubin.
\newblock Causal effects in clinical and epidemiological studies via potential
  outcomes: Concepts and analytical approaches.
\newblock \emph{Annual review of public health}, 21\penalty0 (1):\penalty0
  121--145, 2000.

\bibitem[Malone et~al.(2017)Malone, Kangas, J{\"a}rvisalo, Koivisto, and
  Myllym{\"a}ki]{malone2017asl}
Brandon Malone, Kustaa Kangas, Matti J{\"a}rvisalo, Mikko Koivisto, and Petri
  Myllym{\"a}ki.
\newblock As-asl: {{Algorithm}} selection with auto-sklearn.
\newblock In \emph{Open Algorithm Selection Challenge 2017}. {PMLR}, 2017.

\bibitem[Mena et~al.(2018)Mena, Belanger, Linderman, and
  Snoek]{menaLearningLatentPermutations2018}
Gonzalo Mena, David Belanger, Scott Linderman, and Jasper Snoek.
\newblock Learning {{Latent Permutations}} with {{Gumbel}}-{{Sinkhorn
  Networks}}.
\newblock In \emph{ICLR}, 2018.

\bibitem[Mena et~al.(2020)Mena, Varol, Nejatbakhsh, Yemini, and
  Paninski]{mena2020sinkhorn}
Gonzalo Mena, Erdem Varol, Amin Nejatbakhsh, Eviatar Yemini, and Liam Paninski.
\newblock Sinkhorn permutation variational marginal inference.
\newblock In \emph{AABI}, 2020.

\bibitem[Merkle and Rosseel(2018)]{merkle2018blavaan}
Edgar~C Merkle and Yves Rosseel.
\newblock Blavaan: {{Bayesian}} structural equation models via parameter
  expansion.
\newblock \emph{JOURNAL OF STATISTICAL SOFTWARE}, 85\penalty0 (4):\penalty0
  1--30, 2018.

\bibitem[Mohamed et~al.(2019)Mohamed, Rosca, Figurnov, and
  Mnih]{mohamedMonteCarloGradient2019}
Shakir Mohamed, Mihaela Rosca, Michael Figurnov, and Andriy Mnih.
\newblock Monte {{Carlo Gradient Estimation}} in {{Machine Learning}}.
\newblock \emph{arXiv:1906.10652}, June 2019.

\bibitem[Ng et~al.(2020)Ng, Ghassami, and Zhang]{ngRoleSparsityDAG2020}
Ignavier Ng, AmirEmad Ghassami, and Kun Zhang.
\newblock On the {{Role}} of {{Sparsity}} and {{DAG Constraints}} for
  {{Learning Linear DAGs}}.
\newblock \emph{NeurIPS}, 2020.

\bibitem[Nie et~al.(2016)Nie, {de Campos}, and Ji]{nie2016learning}
Siqi Nie, Cassio {de Campos}, and Qiang Ji.
\newblock Learning {{Bayesian}} networks with bounded tree-width via guided
  search.
\newblock In \emph{AAAI}, 2016.

\bibitem[Pearl(2009)]{pearl2009causality}
Judea Pearl.
\newblock \emph{Causality}.
\newblock {Cambridge university press}, 2009.

\bibitem[Peters and B{\"u}hlmann(2014)]{peters2014identifiability}
Jonas Peters and Peter B{\"u}hlmann.
\newblock Identifiability of {{Gaussian}} structural equation models with equal
  error variances.
\newblock \emph{Biometrika}, 101\penalty0 (1):\penalty0 219--228, 2014.

\bibitem[Peters et~al.(2014)Peters, Mooij, Janzing, and
  Sch{\"o}lkopf]{peters2014causal}
Jonas Peters, Joris~M Mooij, Dominik Janzing, and Bernhard Sch{\"o}lkopf.
\newblock Causal discovery with continuous additive noise models.
\newblock \emph{JMLR}, 2014.

\bibitem[Piironen and Vehtari(2017)]{piironen2017hyperprior}
Juho Piironen and Aki Vehtari.
\newblock On the hyperprior choice for the global shrinkage parameter in the
  horseshoe prior.
\newblock In \emph{AISTATS}, 2017.

\bibitem[Rezende and Mohamed(2015)]{rezende2015variational}
Danilo Rezende and Shakir Mohamed.
\newblock Variational inference with normalizing flows.
\newblock In \emph{ICML}, 2015.

\bibitem[Richardson et~al.(2002)Richardson, Spirtes,
  et~al.]{richardson2002ancestral}
Thomas Richardson, Peter Spirtes, et~al.
\newblock Ancestral graph markov models.
\newblock \emph{The Annals of Statistics}, 2002.

\bibitem[Roeder et~al.(2017)Roeder, Wu, and
  Duvenaud]{roederStickingLandingSimple2017}
Geoffrey Roeder, Yuhuai Wu, and David~K Duvenaud.
\newblock Sticking the landing: Simple, lower-variance gradient estimators for
  variational inference.
\newblock In \emph{NIPS}, 2017.

\bibitem[Sachs et~al.(2005)Sachs, Perez, Pe'er, Lauffenburger, and
  Nolan]{sachsCausalProteinsignalingNetworks2005}
Karen Sachs, Omar Perez, Dana Pe'er, Douglas~A. Lauffenburger, and Garry~P.
  Nolan.
\newblock Causal protein-signaling networks derived from multiparameter
  single-cell data.
\newblock \emph{Science (New York, N.Y.)}, April 2005.

\bibitem[Scanagatta et~al.(2015)Scanagatta, {de Campos}, Corani, and
  Zaffalon]{scanagatta2015learning}
Mauro Scanagatta, Cassio~Polpo {de Campos}, Giorgio Corani, and Marco Zaffalon.
\newblock Learning bayesian networks with thousands of variables.
\newblock In \emph{{{NIPS}}}, 2015.

\bibitem[Schwarz et~al.(1978)]{schwarz1978estimating}
Gideon Schwarz et~al.
\newblock Estimating the dimension of a model.
\newblock \emph{Annals of statistics}, 6\penalty0 (2):\penalty0 461--464, 1978.

\bibitem[Shalit et~al.(2016)Shalit, Johansson, and
  Sontag]{shalitEstimatingIndividualTreatment2016}
Uri Shalit, Fredrik~D. Johansson, and David Sontag.
\newblock Estimating individual treatment effect: Generalization bounds and
  algorithms.
\newblock \emph{arXiv:1606.03976 [cs, stat]}, June 2016.

\bibitem[Sheehan et~al.(2014)Sheehan, Bartlett, and
  Cussens]{sheehan2014improved}
Nuala~A Sheehan, Mark Bartlett, and James Cussens.
\newblock Improved maximum likelihood reconstruction of complex
  multi-generational pedigrees.
\newblock \emph{Theoretical population biology}, 97:\penalty0 11--19, 2014.

\bibitem[Shimizu et~al.(2011)Shimizu, Inazumi, Sogawa, Hyv{\"a}rinen, Kawahara,
  Washio, Hoyer, and Bollen]{shimizu2011directlingam}
Shohei Shimizu, Takanori Inazumi, Yasuhiro Sogawa, Aapo Hyv{\"a}rinen,
  Yoshinobu Kawahara, Takashi Washio, Patrik~O Hoyer, and Kenneth Bollen.
\newblock Directlingam: A direct method for learning a linear non-gaussian
  structural equation model.
\newblock \emph{The Journal of Machine Learning Research}, 12:\penalty0
  1225--1248, 2011.

\bibitem[Sinkhorn(1964)]{sinkhorn1964relationship}
Richard Sinkhorn.
\newblock A relationship between arbitrary positive matrices and doubly
  stochastic matrices.
\newblock \emph{The annals of mathematical statistics}, 35\penalty0
  (2):\penalty0 876--879, 1964.

\bibitem[Spirtes and Glymour(1991)]{spirtes1991algorithm}
Peter Spirtes and Clark Glymour.
\newblock An algorithm for fast recovery of sparse causal graphs.
\newblock \emph{Social science computer review}, 1991.

\bibitem[Spirtes et~al.(1993)Spirtes, Glymour, and
  Scheines]{spirtesCausationPredictionSearch1993}
Peter Spirtes, Clark Glymour, and Richard Scheines.
\newblock \emph{Causation, {{Prediction}}, and {{Search}}}.
\newblock {Springer-Verlag}, 1993.

\bibitem[Teyssier and Koller(2005)]{teyssierOrderingbasedSearchSimple2005}
Marc Teyssier and Daphne Koller.
\newblock Ordering-based search: A simple and effective algorithm for learning
  {{Bayesian}} networks.
\newblock In \emph{UAI}, 2005.

\bibitem[Tibshirani(1996)]{tibshirani1996regression}
Robert Tibshirani.
\newblock Regression shrinkage and selection via the lasso.
\newblock \emph{Journal of the Royal Statistical Society: Series B
  (Methodological)}, 58\penalty0 (1):\penalty0 267--288, 1996.

\bibitem[Tomczak(2016)]{tomczak2016some}
Jakub~M Tomczak.
\newblock On some properties of the low-dimensional gumbel perturbations in the
  perturb-and-map model.
\newblock \emph{Statistics \& Probability Letters}, 115:\penalty0 8--15, 2016.

\bibitem[Tsamardinos et~al.(2006)Tsamardinos, Brown, and
  Aliferis]{tsamardinos2006max}
Ioannis Tsamardinos, Laura~E Brown, and Constantin~F Aliferis.
\newblock The max-min hill-climbing bayesian network structure learning
  algorithm.
\newblock \emph{Machine learning}, 65\penalty0 (1):\penalty0 31--78, 2006.

\bibitem[Varian(2016)]{varian2016causal}
Hal~R Varian.
\newblock Causal inference in economics and marketing.
\newblock \emph{PNaS}, 113, 2016.

\bibitem[Viinikka et~al.(2020)Viinikka, Hyttinen, Pensar, and
  Koivisto]{viinikka2020towards}
Jussi Viinikka, Antti Hyttinen, Johan Pensar, and Mikko Koivisto.
\newblock Towards scalable bayesian learning of causal dags.
\newblock \emph{NeurIPS}, 33, 2020.

\bibitem[Wei et~al.(2020)Wei, Gao, and Yu]{wei2020dags}
Dennis Wei, Tian Gao, and Yue Yu.
\newblock {{DAGs}} with no fears: {{A}} closer look at continuous optimization
  for learning bayesian networks.
\newblock \emph{NeurIPS}, 2020.

\bibitem[Williams(1992)]{williams1992simple}
Ronald~J Williams.
\newblock Simple statistical gradient-following algorithms for connectionist
  reinforcement learning.
\newblock \emph{Machine learning}, 8\penalty0 (3-4):\penalty0 229--256, 1992.

\bibitem[Yu et~al.(2019)Yu, Chen, Gao, and Yu]{yuDAGGNNDAGStructure2019}
Yue Yu, Jie Chen, Tian Gao, and Mo~Yu.
\newblock {{DAG}}-{{GNN}}: {{DAG Structure Learning}} with {{Graph Neural
  Networks}}.
\newblock \emph{arXiv:1904.10098 [cs, stat]}, April 2019.

\bibitem[Zheng et~al.(2018)Zheng, Aragam, Ravikumar, and
  Xing]{zhengDAGsNOTEARS2018}
Xun Zheng, Bryon Aragam, Pradeep Ravikumar, and Eric~P Xing.
\newblock Dags with no tears: Continuous optimization for structure learning.
\newblock In \emph{NeurIPS}, 2018.

\bibitem[Zhuang et~al.(2020)Zhuang, Tang, Tatikonda, Dvornek, Ding,
  Papademetris, and Duncan]{zhuang2020adabelief}
Juntang Zhuang, Tommy Tang, Sekhar Tatikonda, Nicha Dvornek, Yifan Ding,
  Xenophon Papademetris, and James~S Duncan.
\newblock Adabelief optimizer: {{Adapting}} stepsizes by the belief in observed
  gradients.
\newblock In \emph{NeurIPS}, 2020.

\end{thebibliography}
\section*{Checklist}
% %%% BEGIN INSTRUCTIONS %%%
% The checklist follows the references.  Please
% read the checklist guidelines carefully for information on how to answer these
% questions.  For each question, change the default \answerTODO{} to \answerYes{},
% \answerNo{}, or \answerNA{}.  You are strongly encouraged to include a {\bf
% justification to your answer}, either by referencing the appropriate section of
% your paper or providing a brief inline description.  For example:
% \begin{itemize}
%   \item Did you include the license to the code and datasets? \answerYes{See Section~\ref{gen_inst}.}
%   \item Did you include the license to the code and datasets? \answerNo{The code and the data are proprietary.}
%   \item Did you include the license to the code and datasets? \answerNA{}
% \end{itemize}
% Please do not modify the questions and only use the provided macros for your
% answers.  Note that the Checklist section does not count towards the page
% limit.  In your paper, please delete this instructions block and only keep the
% Checklist section heading above along with the questions/answers below.
% %%% END INSTRUCTIONS %%%

\begin{enumerate}

\item For all authors...
\begin{enumerate}
  \item Do the main claims made in the abstract and introduction accurately reflect the paper's contributions and scope?
    \answerYes{}
  \item Did you describe the limitations of your work?
    \answerYes{Yes, in the conclusion we discuss how we are limited to linear relationships. We also emphasise throughout the paper that in many settings, the true DAG parameters are nonidentifiable from observational data}
  \item Did you discuss any potential negative societal impacts of your work?
    \answerYes{In the conclusion we discuss how wrongly assuming linear relationships could lead to incorrect inferences, and how those could be amplified in minority groups}
  \item Have you read the ethics review guidelines and ensured that your paper conforms to them?
    \answerYes{We conform to the ethics review guidelines}
\end{enumerate}

\item If you are including theoretical results...
\begin{enumerate}
  \item Did you state the full set of assumptions of all theoretical results?
    \answerYes{We give the full preconditions for theorem 1 to hold in the supplementary material}
	\item Did you include complete proofs of all theoretical results?
          \answerYes{We give a full proof of theorem 1 in the supplementary material}
\end{enumerate}

\item If you ran experiments...
\begin{enumerate}
  \item Did you include the code, data, and instructions needed to reproduce the main experimental results (either in the supplemental material or as a URL)?
    \answerNo{Given time constraints, we were not able to get the code in shape to release}
  \item Did you specify all the training details (e.g., data splits, hyperparameters, how they were chosen)?
    \answerYes{In the additional material we discuss all hyperparameters, including for the cross-validation for GOLEM}
	\item Did you report error bars (e.g., with respect to the random seed after running experiments multiple times)?
    \answerYes{Yes, all our results have error bars from running with several random seeds}
	\item Did you include the total amount of compute and the type of resources used (e.g., type of GPUs, internal cluster, or cloud provider)?
    \answerYes{This is described in the appendix, section I}
\end{enumerate}

\item If you are using existing assets (e.g., code, data, models) or curating/releasing new assets...
\begin{enumerate}
  \item If your work uses existing assets, did you cite the creators?
    \answerYes{Yes, we used the data producing code from Ng et al and they are cited}
  \item Did you mention the license of the assets?
    \answerYes{We described the license}
  \item Did you include any new assets either in the supplemental material or as a URL?
    \answerNo{}
  \item Did you discuss whether and how consent was obtained from people whose data you're using/curating?
    \answerNA{Under the license there is no need to get consent from the github maintainers as long as we properly cite}
  \item Did you discuss whether the data you are using/curating contains personally identifiable information or offensive content?
    \answerNA{}
\end{enumerate}

\item If you used crowdsourcing or conducted research with human subjects...
\begin{enumerate}
  \item Did you include the full text of instructions given to participants and screenshots, if applicable?
    \answerNA{}
  \item Did you describe any potential participant risks, with links to Institutional Review Board (IRB) approvals, if applicable?
    \answerNA{}
  \item Did you include the estimated hourly wage paid to participants and the total amount spent on participant compensation?
    \answerNA{}
\end{enumerate}

\end{enumerate}

\appendix
% NOTE: necessary when ptmx or no mathfont class option is given
\providecommand{\upGamma}{\Gamma}
\providecommand{\uppi}{\pi}

\newpage
\onecolumn
\appendix
\section{Marginalizing Gaussian Edge Weights}
If we set the prior over edge weights to be a \(d(d-1)/2\)-dimensional isotropic Gaussian with standard deviation \(\nu\), so \(l \sim \mathcal{N}(0, \Sigma_l)\) with \(\Sigma_l = I\nu^2\), then we can analytically marginalize out the weights in the
likelihood expression below.
For an arbitrary datapoint $X$, we have
\begin{align}
  \label{eq:1}  
  p(X|P, \Sigma)= C \int_{\R^{d(d-1)/2}} \sqrt{\det \Theta}\exp\left(-\frac{X^\top\Theta X + \nu^{-2}l^\top l}{2}\right) dl,                              
\end{align}
with \(\Theta  = {(I - PLP^\top)}^\top\operatorname{diag} (\sigma^{-2}) (I - PLP^\top)\) as discussed above, and \(L\in \R^{d\times d}\) a strictly lower triangular matrix with the \(d(d-1)/2\) lower entries forming the vector \(l\). Here, \(C = {(2\pi)}^{-d(d-3)/4} \nu^{-d(d-1)/2}\) is a constant term. 
The following theorem provides a closed-form expression for this integral.
\begin{theorem}
  The log-likelihood \(\log P(X |P, \Sigma)\) as defined in equation~\eqref{eq:1} is
  \begin{align}
    \label{eq:2}
    \log p(X|\Sigma, \nu, P) =  - & \frac{d}{2}\log 2\pi -\sum_{j=1}^d \log \sigma_j - (d-1)\log \nu + \frac{1}{2} X^\top (S -  \Sigma) X  - \frac{1}{2}\log \det D' \nonumber \\
    - & \frac{1}{2} \log \det (\Sigma'^{-1}+D'^{-1}\nu'^{-2}),
  \end{align}
  with \(D = \operatorname{diag}[PGP^\top X^2]\) where \(G\) is the strictly lower-triangular matrix with all lower-triangular entries equal to \(1\). \(D\) has (almost surely) one zero diagonal entry, with index \(i\). Primed matrices ($\Sigma'$, $D'$) are in \(\R^{d-1 \times d-1}\), with the \(i\)th diagonal entry removed. Finally, \(S \in \R^{d\times d}\) is a diagonal matrix with \(S_{jj} = \frac{\nu^2D_{jj}}{\sigma_j^2(\sigma_j^2 + \nu^2 D_{jj})}\). 
\end{theorem}
\begin{proof}\label{sec:marginalization-proof}
  
In the case where the edge weights are known to be drawn from a \(d(d-1)/2\)-dimensional isotropic Gaussian, \(l \sim \mathcal{N}(0, I\nu^2)\), then we are able to analytically marginalize out the weights in the term \(\Ex{l \sim \mathcal{N} }{p(X|L,P, \Sigma)}\). We have
\begin{align*}
p(X|P, \Sigma, \nu) = \int_{l} p(X|L, P, \Sigma) P(L|\nu)dl.
\end{align*}

We can compute this analytically. We write \(l\) for the vector of \(d(d-1)/2\) strictly lower-triangular
elements, and \(L\) for the matrix with these as the lower-triangular elements. We then have 

\begin{align*}
p(X|P, \Sigma, \nu) = \int_l \frac{1}{{(2\pi)}^{d/2} {(\det \Sigma)}^{1/2}} \frac{1}{{(2\pi)}^{d(d-1)/4} \nu^{d(d-1)/2}} \exp\left(-\frac{1}{2}X^\top \Theta X\right)\exp\left(-\frac{1}{2\nu^2}l^\top l\right) dl.
\end{align*}
assuming that the mean of \(l\) and \(X\) are zero. 
Now 
\begin{align*}
\Theta  = (I - W)\operatorname{diag} (\sigma^{-2}) {(I - W)}^\top,
\end{align*}
or 
\begin{align*}
\Theta  = {(I - PLP^\top)}^\top\operatorname{diag} (\sigma^{-2}) (I - PLP^\top),
\end{align*}
with 
\begin{align*}
W^\top = PLP^\top,
\end{align*}
(note that this is different from the \(W = PLP^\top\) parameterisation in the main text, but is equivalent)
and a vector of noise standard deviations \(\sigma\). We have \(\det \Theta = 1/\det \Sigma\) because 
\(\det (I - PLP^\top) = \det (PIP^\top - PLP^\top) = \det P(I - L)P^\top = 1\). 

Writing \(\Sigma\) for the diagonal matrix \(\operatorname{diag}(\sigma^2)\), if we expand out this matrix product we get 
\(\Theta  = \Sigma^{-1} - PL^\top P^\top \Sigma^{-1} PLP^\top - PL^\top P^\top \Sigma^{-1} - {(PL^\top P^\top \Sigma^{-1})}^\top\). Overall then, inside the exponential we have 
\(X^\top\Theta X  = X^\top\Sigma^{-1}X - X^\top PL^\top P^\top \Sigma^{-1} PLP^\top X - 2X^\top PL^\top P^\top \Sigma^{-1}X\). 
Since \(L\) is the integration variable, the first term (not containing \(L\)) can be taken outside the integral.

We now note an interesting fact about the \(X^\top PL^\top P^\top \Sigma^{-1} PLP^\top X - 2X^\top P L^\top P^\top \Sigma^{-1}X\) terms
in the integral. Although we are integrating over the \(d(d-1)/2\)-dimensional space of \(l\) values, this term
is only \(d\)-dimensional. For example, in the \(3\times 3\) case, with identity permutation, we have 
\begin{align*}
L = \begin{pmatrix} 0 & 0 & 0 \\ l_1 & 0 & 0 \\ l_2 & l_3 & 0 \end{pmatrix}.
\end{align*}
and so 
\begin{align*}
PLP^\top X  = \begin{pmatrix} 0 & 0 & 0 \\ l_1 & 0 & 0 \\ l_2 & l_3 & 0 \end{pmatrix} \begin{pmatrix} x_1 \\ x_2 \\ x_3 \end{pmatrix} = \begin{pmatrix} 0 \\ l_1 x_1 \\ l_2 x_1 + l_3 x_2 \end{pmatrix},
\end{align*}
where \(x_1, \ldots, x_d\) are the \(d\) individual components of the vector \(X\).
This means that if we are to do a change of variables into the basis
\begin{align*}
u = \begin{pmatrix} l_1 x_1 \\ l_2x_1 + l_3 x_2 \\ l_2 x_1 - l_3 x_2 \end{pmatrix}
\end{align*}
then along the third coordinate direction the \(PLP^\top X\) terms will not change, since it's orthogonal to the terms involved in the integrand.

So we now change variables from \(l\) to \(u\), where the first \(d-1\) components of \(u\) are given by (the \(d-1\) non-identically-zero components of)   \(PLP^\top X\), and the remaining \(d(d-1)/2 - (d-1)\) components are vectors orthonormal to those directions. We write \(u'\) for the \(d-1\)-dimensional vector formed from the nonzero components of \(PLP^\top X\).

Now since the first components of \(u\) are not unit length, we get a term in the \(dl\) differential element corresponding to the 
length of the vectors, which would be e.g. \(\left(x_1, \sqrt{x_1^2 + x_2^2}\right)\) in the case above. Given that we're substituting \(u\) for \(l\) we also will need to rescale
the \(\Sigma_l\) term. 

We can express this via a diagonal matrix \(D_{xP}\) which we can construct as
\begin{align*}
D = \operatorname{diag}\left(P GP^\top X^2\right),
\end{align*}
where \(G\) is the matrix with a \(1\) in the strictly lower-triangular entries and \(0\) otherwise.

We now have 
\begin{align*}
C(X, P)\int_{u'} \exp\left(-\frac{1}{2}\left(u'^\top (\Sigma^{'-1} + D^{'-1} \nu^{-2})u' - 2u'\Sigma^{'-1}X'\right)\right) du',
\end{align*}
with a constant \(C(X, P) = \frac{1}{{(2\pi)}^{d/2} {(\det \Sigma)}^{1/2}} \frac{1}{{(2\pi)}^{(d-1)/2} \nu^{d-1}}\exp\left(-\frac{1}{2}X^\top \Sigma X\right)\frac{1}{{(\det D')}^{1/2}}\),
and we have integrated out the \(d(d-1)/2 - (d-1)\) dimensions of \(u\) not involved in the term with \(\Theta\). Primed matrices (e.g. \(\Sigma' \in \R^{(d-1)\times(d-1)}\)) have had the column and row corresponding to the zero entry of \(D\) removed. 
We can then use the identity 
\begin{align*}
\int_x \exp \left[-\frac{1}{2} \mathbf{x}^{T} \mathbf{A} \mathbf{x}+\mathbf{c}^{T} \mathbf{x}\right] d \mathbf{x}=\sqrt{\det\left(2 \pi \mathbf{A}^{-1}\right)} \exp \left[\frac{1}{2} \mathbf{c}^{T} \mathbf{A}^{-T} \mathbf{c}\right]
\end{align*}

where here we have \(A = (\Sigma^{'-1} + D'^{-1}\nu^{-2})\) and \(c = \Sigma^{'-1}x'\), we get 

\begin{align*}
&\int_{u'} \exp\left(-\frac{1}{2}\left(u'^\top (\Sigma^{'-1} + D^{'-1} \nu^{-2})u' - 2u'\Sigma^{'-1}X'\right)\right) du'\\ = &{(2\pi)}^{(d-1)/2}{\left(\det \left(\Sigma'^{-1} + D'^{-1}\nu^{-2}\right)\right)}^{-1/2}\exp\left(\frac{1}{2}{(\Sigma'^{-1}X')}^\top {(\Sigma'^{-1} + D^{-1}\nu^{-2})}^{-1}\Sigma'^{-1}X'\right),
\end{align*}
We end up with 

\begin{align*}
\frac{1}{{(2\pi)}^{d/2} {(\det \Sigma)}^{1/2}} \frac{1}{ \nu^{d-1}}& \exp\left(-\frac{1}{2}X^\top \Sigma X\right)\frac{1}{{(\det D')}^{1/2}}\\  &\times {\left(\det \left(\Sigma'^{-1} + D'^{-1}\nu^{-2}\right)\right)}^{-1/2}\exp\left(\frac{1}{2}{(\Sigma'^{-1}X')}^\top {(\Sigma'^{-1} + D^{-1}\nu^{-2})}^{-1}\Sigma'^{-1}X'\right).
\end{align*}

We can also write the term \({(\Sigma'^{-1}X')}^\top {(\Sigma'^{-1} + D^{-1}\nu^{-2})}^{-1}\Sigma'^{-1}X'\) as \(XSX^\top\) with a diagonal matrix \(S\), where

\begin{align*}
S_{ii} = \frac{\nu^2d_i}{\sigma_i^2(\sigma_i^2 + \nu^2 d_i)},
\end{align*}

with \(d_i\) the \(i\)th diagonal entry of \(D\), so

\begin{align*}
p(X|\Sigma, \nu, P) = \frac{1}{{(2\pi)}^{d/2} \prod_i \sigma_i} \frac{1}{ \nu^{d-1}}\exp\left(-\frac{1}{2}X^\top \Sigma X\right)\frac{1}{{(\det D')}^{1/2}}  & {\left(\det \left(\Sigma'^{-1} + D'^{-1}\nu^{-2}\right)\right)}^{-1/2}\\ 
& \times \exp\left(\frac{1}{2}X^\top S X\right).
\end{align*}

This gives a log-likelihood

\begin{align*}
\log p(X|\Sigma, \nu, P) = &-\frac{d}{2}\log 2\pi -\sum_i \log \sigma_i - (d-1)\log \nu\\ - &\frac{1}{2}X^\top \Sigma X - \frac{1}{2}\log \det D' - \frac{1}{2} \log \det (\Sigma'^{-1}+D'^{-1}\nu^{-2}) + \frac12 X^\top S X.
\end{align*}
\end{proof}

\section{Additional Experimental Details}\label{sec:addit-exper-deta}
\textbf{Baseline GOLEM settings:}  As discussed in the main text, we found that the suggested~\cite{ngRoleSparsityDAG2020} regularization hyperparameter of \(\lambda_1 = 2\cdot10^{-2}\), led to poor performance. This poor performance at low \(n\) is described in appendix G of~\cite{ngRoleSparsityDAG2020}. However, we found that varying \(\lambda_1\) could increase performance. We therefore hold out one-fifth of the data when training GOLEM, and use this as a validation set for cross-validation over \(\lambda_1 \in \left\{2\cdot 10^{-2}, 2\cdot 10^{-1}, 2, 20\right\}\), scoring fit based on the sample mean-squared error \(X - W^\top X\). GOLEM has better performance when trained for longer than the \(10^5\) steps suggested, and so we trained for \(2\cdot 10^5\) steps, or until the gradient norm was lower than \(0.01\). Additionally, we tested a probabilistic formulation of GOLEM by training an ensemble of GOLEM models on bootstrapped dataset. We used 20 models for the ensemble, excluding 5\% of the dataset in each bootstrap dataset. We found this method was not competitive with the other baselines (e.g. SHD of 150 on the 32-dimensional ER2 graph), and didn't evaluate this baseline further due to its extremely large computational cost. 

\textbf{Implementation details:} In all experiments we use a two-layer multilayer perceptron with \(128\) hidden units to parameterize the function returning the parameters \(T\) of the Gumbel-Sinkhorn in \(q(P|L, \Sigma)\).
We use the adabelief optimizer~\cite{zhuang2020adabelief} with a step size \(10^{-3}\).
We also apply a scaled sigmoid to the logits produced by the MLP so that they are in the range \((-20, 20)\) to increase stability of training.
Nonetheless, the training can still be unstable at times, with sudden drops in the ELBO.
During training, we found it useful to monitor the KL divergence between \(\mathcal{N}(0, {\hat \Theta}^{-1})\) and \(\mathcal{N}(0, \Theta_\phi^{-1})\), where \(\hat \Theta\) is the empirical precision of the training data and \(\Theta_\phi\) is the average precision matrix described in section 2.2, generated by our parameterized posterior distribution.
We found that this quantity correlates strongly with the sample Hamming distance, and a sudden increase can indicate optimization issues during training.
In the synthetic data cases, we used the code\footnote{\url{https://github.com/ignavier/golem}, Apache 2.0 license} provided by the authors of~\cite{ngRoleSparsityDAG2020} to generate the data and compute the structural Hamming distance.
Due to the large number of Sinkhorn steps required, we used the implicit derivative based on the adjoint method, provided by the optimal transport tools library\footnote{\url{https://github.com/google-research/ott}, Apache 2.0 license}.
This avoids the high memory cost of differentiating through the computation graph of the Sinkhorn iterates. 
% cite not strictly necessary \cite{chen2016training}
We set \(\tau\) for the horseshoe prior as \(2 / (d\sqrt{n})\) for all the experiments, since we expect average degree on the order of 1. The structural Hamming distance reported for our predictions, is computed via sampling from the posterior distribution and averaging the SHD induced by all samples.

\textbf{Sinkhorn Implementation Details:} To ensure we stick closely to the set of permutation matrices, we use a relatively small value of \(\tau\), 0.2. We also use the straight-through gradient estimator \citep{bengio2013estimating}. This means that on the forward pass of the backpropogation algorithm, we use the \(\tau \to 0\) limiting value of the Sinkhorn algorithm. On the backward pass the gradients are taken with respect to the finite-\(\tau\) doubly-stochastic matrix. To obtain the \(\tau = 0\) output we use the Hungarian algorithm \citep{kuhn1955hungarian}. Compared to previous work using Gumbel-Sinkhorn \citep{menaLearningLatentPermutations2018, li2021discovering}, we note that we require a large number of Sinkhorn iterations for stable training, 
sometimes requiring over 1000 iterations to ensure the row and column sums equal 1 within a tolerance of \(0.01\).

\begin{figure}[h]
  \centering
  \includegraphics[width=1.1\textwidth]{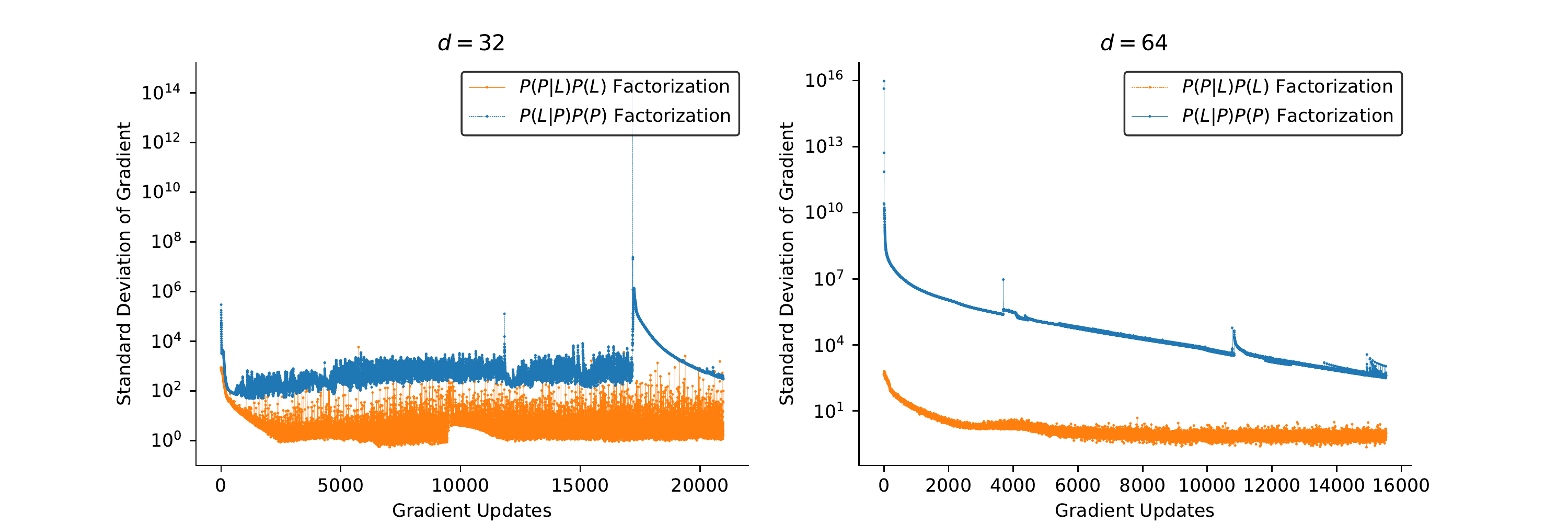}
  \caption{The variance of the gradient when training with two alternative factorizations. We note that the factorization where \(P\) is sampled first has much higher gradient variance, leading to unstable and difficult optimization.}\label{fig:variance}
\end{figure}

\section{Gradient Variance}\label{app:gradient_variance}

In the main paper we discuss the possibility of an alternate factorization of the probabilistic model, \(q_\phi(L, \Sigma|P)q(P)\) where we generate $P$ first, followed by \(L\) and \(\Sigma\).
However, empirically we find that the factorization order in equation~\eqref{eq:factorized_order} performs much better than the alternative factorization, 
with orders of magnitude lower variance in the gradient (Figure~\ref{fig:variance}). 
This suggests that the bias-variance trade-offs due to stochastic optimization in performing VI over $P$ (relaxations to discrete distribution over permutations, see Section~\ref{sec:perm_relax}) are more challenging than $(L, \Sigma)$ (continuous r.v.\ with de-facto pathwise gradient estimators, see Section~\ref{sec:l_var_family}). 
\section{Additional Experimental Results}
\subsection{Additional Prediction Statistics}\label{sec:addit-pred-stat}
In this section, we give the true positive rate, false positive rate, and false discovery rate of our algorithm and all the baselines. They are shown in figure~\ref{fig:fdr}.
We observe that we typically perform better on all the metrics, with the exception of the false positive rate for the gumbel cases compared to DLiNGAM. As the DLiNGAM method is designed for the gumbel case while \name{} use a Gaussian likelihood, this makes sense. 
\begin{figure}[h]
  \centering
  \includegraphics[width=1.0\textwidth]{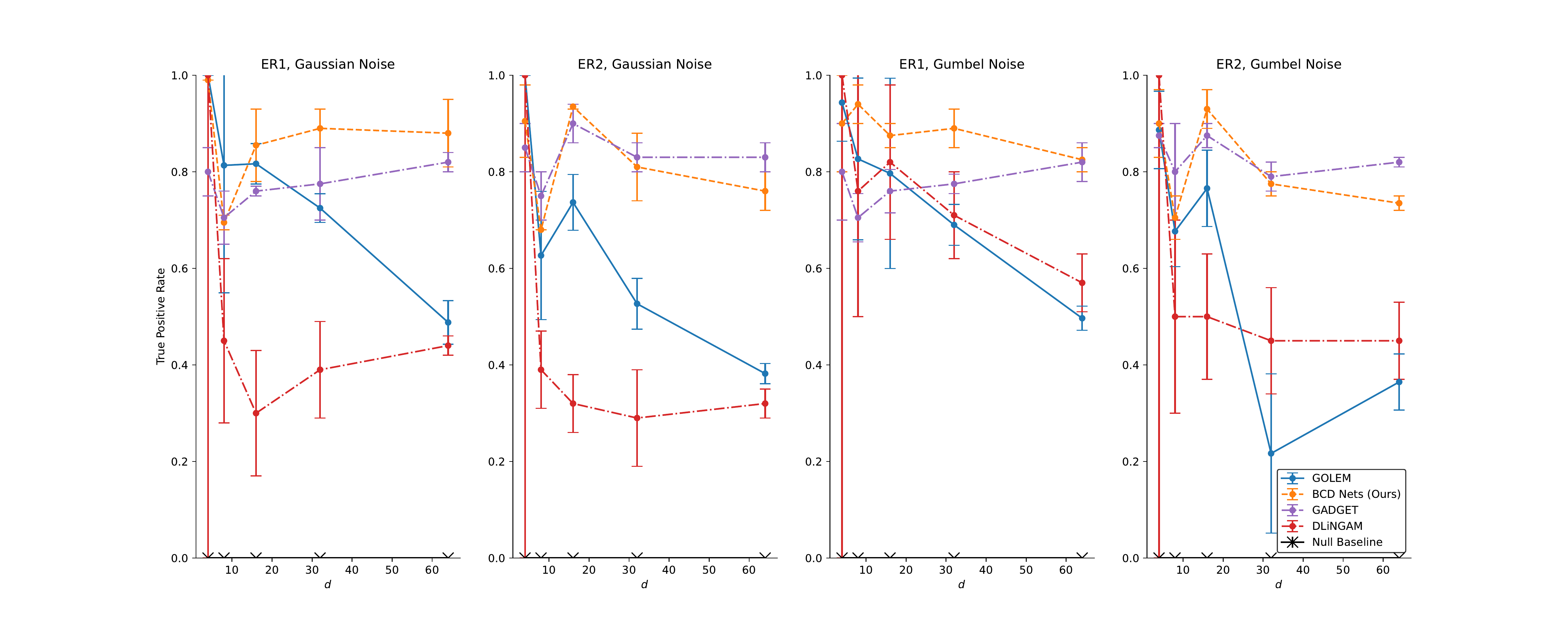}
  \includegraphics[width=1.0\textwidth]{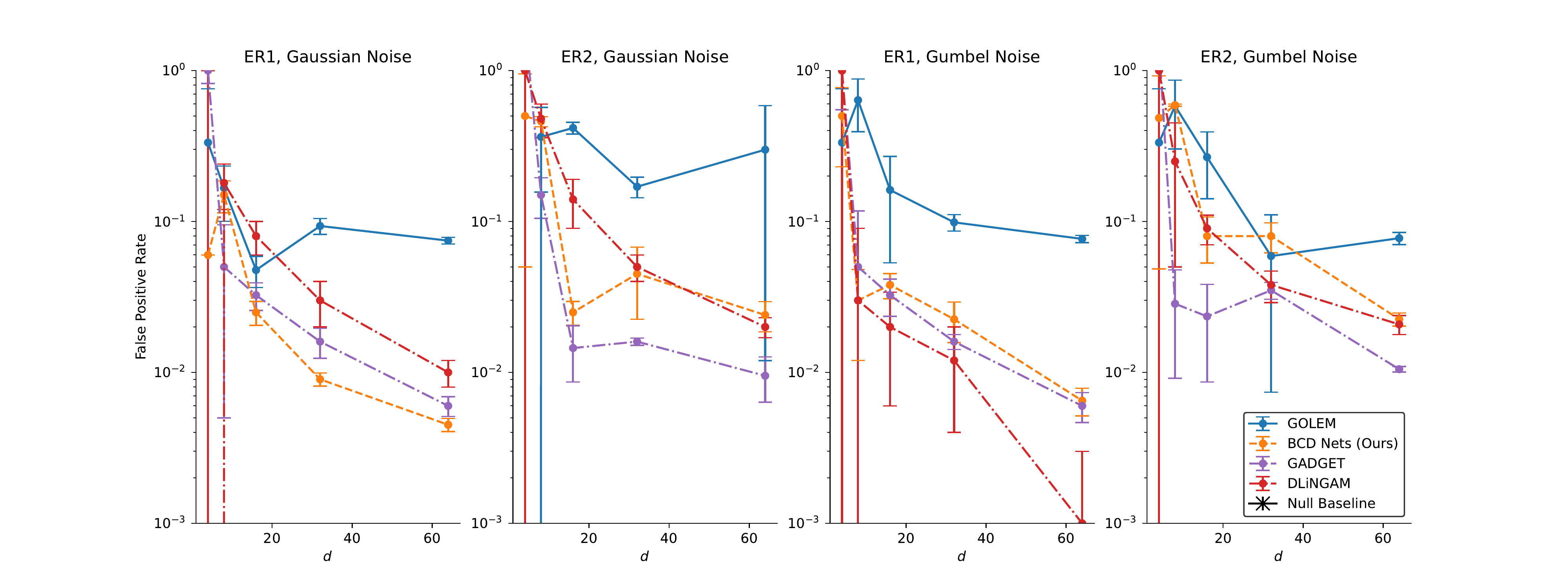}
  \includegraphics[width=1.0\textwidth]{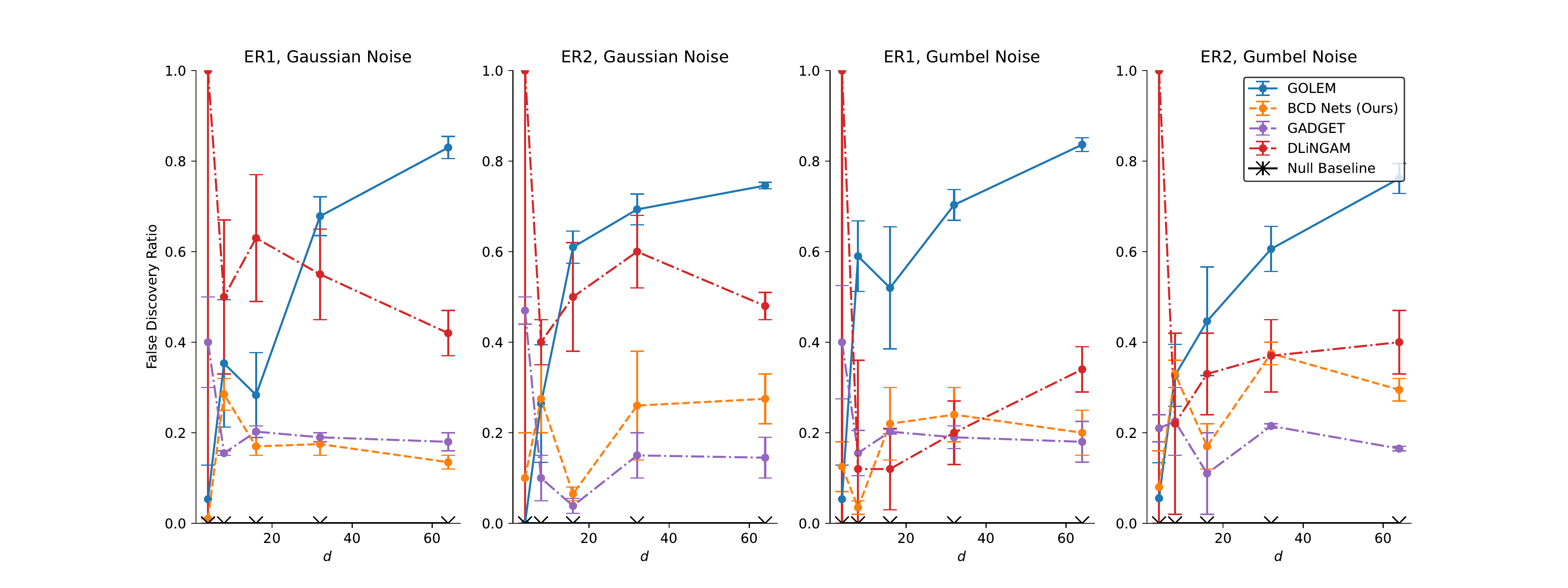}
  \caption{True positive rate, false positive rate, and false discovery rate for the experiments described in Section~\ref{sec:comp-maxim-likel}}\label{fig:fdr}
\end{figure}
\subsection{Qualitative Analysis}\label{sec:qualitative-analysis}
As a brief qualitative analysis the performance of our approach, we consider the three-dimensional case, where we are able to enumerate the \(3! = 6\) permutation matrices. We draw a set of \(n\) data points and train our model until convergence, in the equal-variance setting (ensuring a unique solution). We then draw one hundred samples from the joint distribution \(q(P, L, \Sigma)\), and plot a histogram of the estimated distribution of \(P\), shown in figure~\ref{fig:illustrative}.
\begin{figure}[h]
  \centering
  \includegraphics[width=0.4\textwidth]{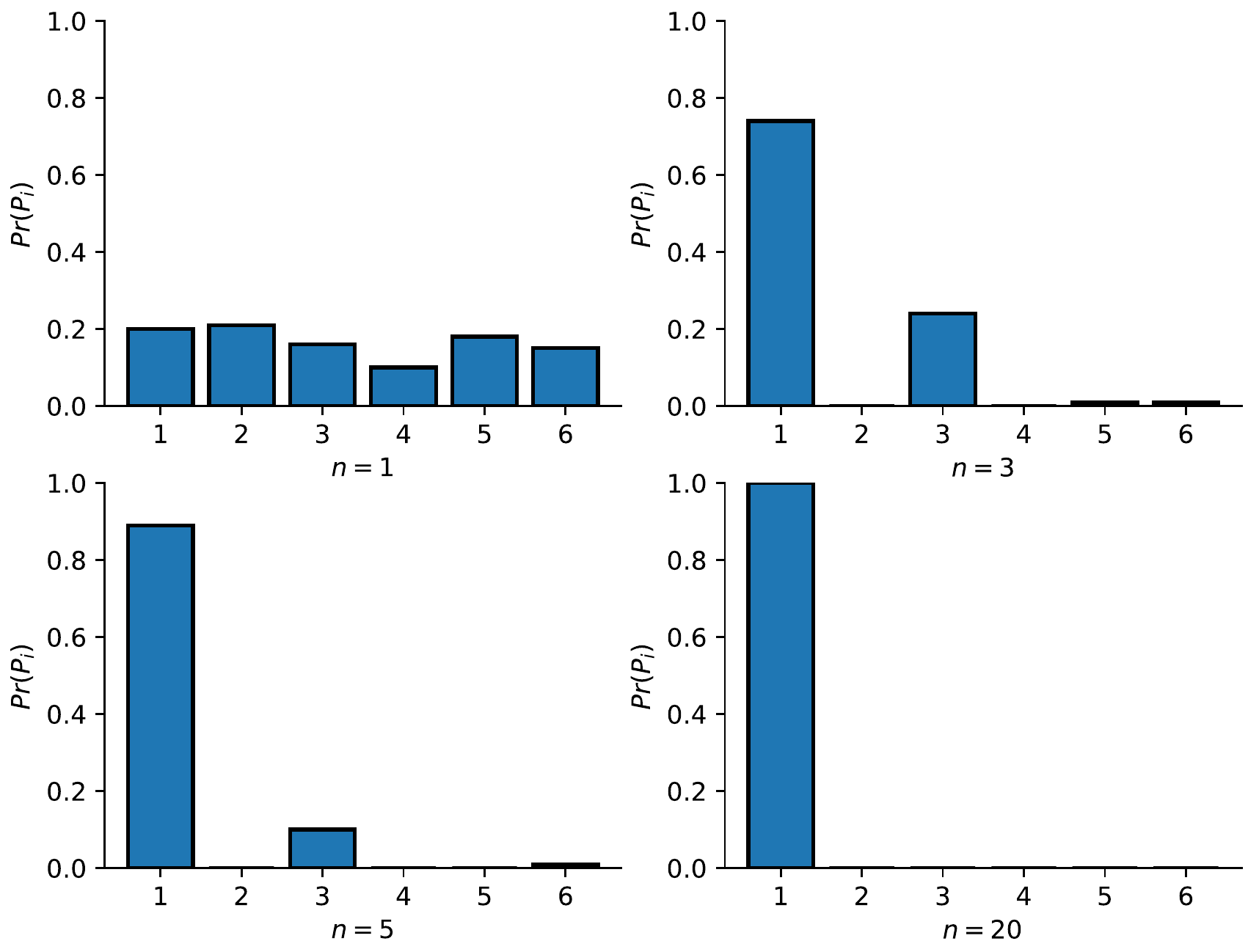}
  \caption{Concentration of the posterior distribution over the $3!=6$ possible permutations for \(n\) samples from a three-dimensional DAG, where the data-generating permutation is \(P_1\). The posterior quickly concentrates onto two candidates, then onto the ground-truth permutation.}\label{fig:illustrative}
\end{figure}
As intuitively expected, the posterior is quite diffuse for small amount of data, but quickly peaks around two possible candidate graphs. With additional data the ground-truth graph is picked out.

\section{Ablation Tables}
Here we give the ablation results when we take away parts of the algorithm. The results are show in tables~\ref{tab:1} and~\ref{tab:2}, demonstrating a degredation in performance when the architectural choices described in the main text are changed. 

\begin{table}[t]
\centering
\caption{Degradation of performance when removing components of our equal-variance model. Reducing the number of Sinkhorn iterations and using a mean-field posterior drastically degrade the solution quality.
  Using a Laplace prior instead of horseshoe modestly worsens performance w.r.t. KL divergence, but quite negatively impacts the SHD.\@
  }\label{tab:1}
\begin{tabular}{lll}
             & SHD & Sample-KL \\ \toprule
  \name{}-EV   &  11 \(\pm\)1   & 6.2 \(\pm\) 0.3          \\
  \midrule
Mean-Field             &  27 \(\pm\) 3 &  21 \(\pm\) 2         \\
Laplace      &  34 \(\pm\) 1  & 8 \(\pm\) 1          \\
Sinkhorn-100 &  38.5 \(\pm 0.5\)  & 73 \(\pm\) 7        
\end{tabular}
\end{table}

\begin{table}[]
\centering
\caption{Degradation of performance when removing components of our approach, for the non-equal-variance model. We see very similar behaviour to the equal-variance case.}\label{tab:2}
\begin{tabular}{lll}
  & SHD & Sample-KL \\ \toprule
\name{}-NV   & \(9 \pm 3\) & \(7 \pm 3\)          \\
  \midrule
Mean-Field   & \(33 \pm 1\)    & \(35.0 \pm 0.5\)          \\
Laplace      &  \(34.5 \pm 0.5\)   & \(9 \pm 1\)          \\
Sinkhorn-100 &  \(31 \pm 1\) & \(56 \pm 11\)        \\
\end{tabular}
\end{table}

\subsection{Causal Inference Experiment Results}\label{sec:caus-infer-exper}
Here we give full results of the causal inference experiment described in Section~\ref{sec:causal-inference}. Figure~\ref{fig:eid} shows the Wasserstein distance between the true interventional distribution and the estimated one, which we call the `Estimated Intervention Distance'. This is averaged over random choices of the edge to intervene on, as well as over random seeds. We observe that we are able to effectively estimate effects of interventions. The degredation of performance in GOLEM likely arises due to the method incorrectly assigning edge weights at higher \(d\) (see e.g.\ the much higher false positive rate of GOLEM compared to \name{} in figure~\ref{fig:fdr}) and so predicting a highly incorrect interventional distribution. 

\begin{figure}[h]
  \centering
  \includegraphics[width=0.9\textwidth]{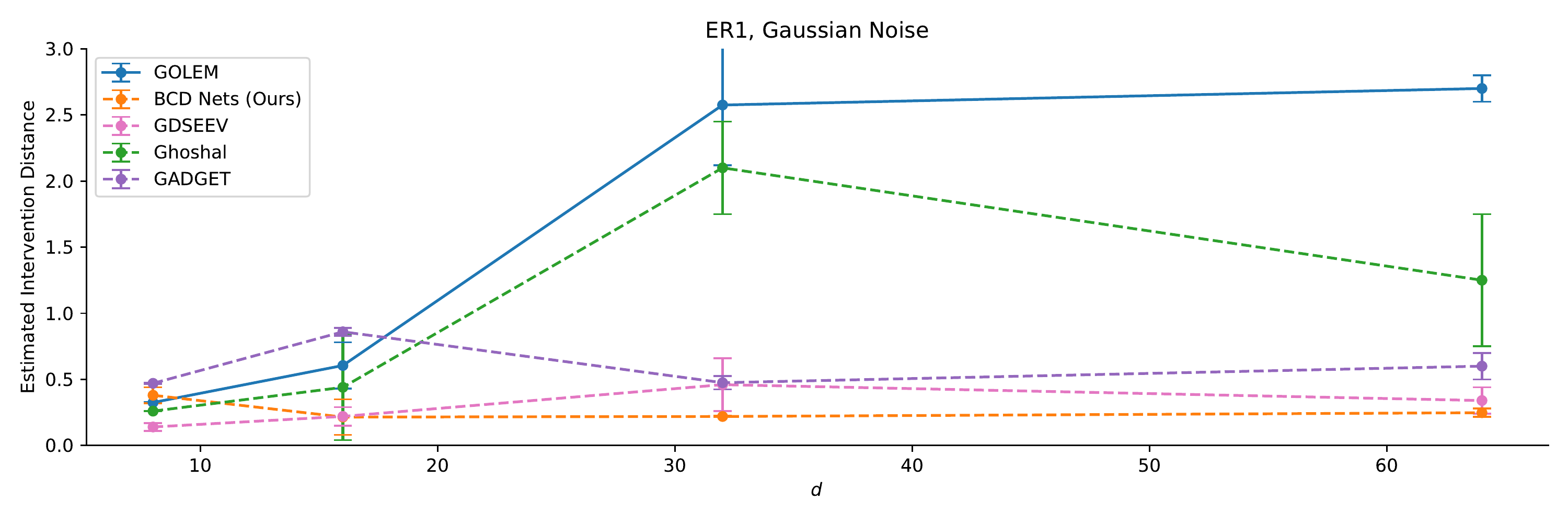}
  \caption{The average Wasserstein distance between the estimated distribution from an intervention and the true distribution from the intervention. Lower is better.}\label{fig:eid}

\end{figure}

\subsection{Variation of \(p\)}
In this section, we analyse the performance of the methods as the degree \(p\) increases. As we expect, the performance of all the methods decreases substantially as the degree increases. The full results are shown in figure \ref{fig:high-p}. We see that when the degree is 4, all the models do worse than chance. This is not too surprising as inferring the presence of 256 edges with only 100 data points is a very challenging task.

\begin{figure}
  \centering
  \includegraphics[width=0.9\textwidth]{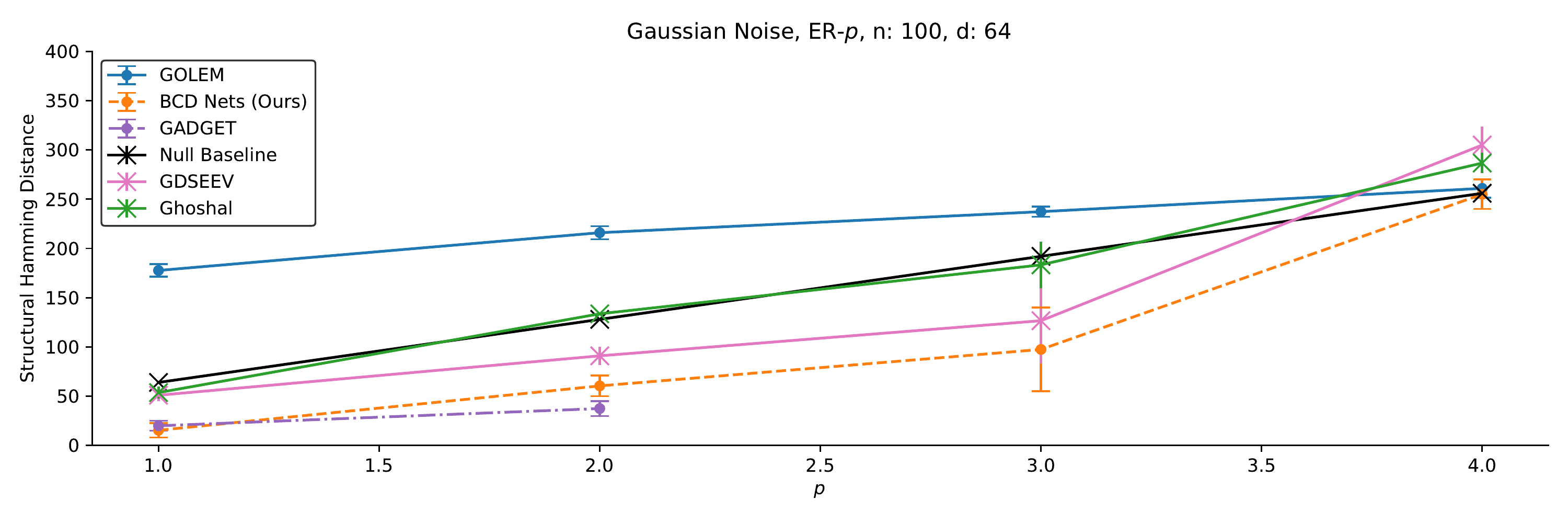}
  \caption{Structural Hamming Distance as a function of the degree \(p\), for the inferred graph compared to the true graph, on random ER-1, 64-dimensional graphs with 100 samples. We see that with large enough degree, all the methods perform poorly.}\label{fig:high-p}
\end{figure}

\newpage
\section{Recovery of Ground Truth for Posterior}\label{sec:recov-ground-truth}
Here we discuss the asymptotic behavior of the posterior distribution for the SEM parameters under an infinite amount of data which has been generated from a linear Gaussian generative process with true parameters \(\Sigma^*, W^*\). We sketch an argument for why we would recover the ground-truth parameters (or a certain equivalence class of the ground truth) under infinite data.

The posterior density for \(\Sigma, W\) under an observation of \(n\) data points \(X_1^n\), is
\begin{align}
  P(\Sigma, W|X_1^n) \propto {{P(X_1^n|\Sigma, W)P(\Sigma, W)}},
\end{align}
which we can write as a sum over data points
\begin{align}\label{eq:7}
  \log P(\Sigma, W|X_1^n) & = \sum_i \log P(X_i|\Sigma, W) + \log P(\Sigma, W),
\end{align}
where the likelihood is given by equation~\eqref{eq:3}. We now argue that the posterior concentrates around the set of SEM parameters quasi-equivalent to the ground-truth parameters, for quasi-equivalence defined in~\cite{ngRoleSparsityDAG2020}. We assume that the prior has support over the true parameters. Now, we assume that the posterior density concentrates around the set of maximum a posteriori (MAP) points. In~\cite{ngRoleSparsityDAG2020}, the problem
\begin{align}\label{eq:5}
  \operatorname{argmax}_{\Sigma, W \in \mathcal{W}}\left\{\sum_i^n \log P(X_i|\Sigma, W) + n \lambda R_{\text{sparse}}(W)\right\},
\end{align}
is studied, where \(\mathcal{W}\) is the set of DAGs, \(R_{\text{sparse}}\) is a regularizer encouraging sparsity, \(\lambda\) a chosen parameter, and \(n \to \infty\). The likelihood is the same as ours.
The solution is a set of \(W\)s which have corresponding DAGs \(G\), which are quasi-equivalent to the true DAG \(G^*\). The corresponding edge weights and noise variances are simply a (regularized) linear regression problem when conditioned on the graph structure, so have a Gaussian likelihood given \(G\). The conditions on \(\lambda\) in~\cite{ngRoleSparsityDAG2020} are not particularly well defined, only requiring ``weights for regularization terms such that the likelihood term dominates asymptotically''.

In order to match equation~\eqref{eq:5} with equation~\eqref{eq:7} and show that the MAP points of our distribution are quasi-equivalent to the true parameters, we would have to use a sparsity-encouraging prior and assume that we could choose \(\lambda = 1/n\). However, this would mean the likelihood grows while the regularizer stays constant.
By analogy to the BIC score~\cite{schwarz1978estimating}, we might want to choose a \(\lambda\) such that the regularization term grows modestly with \(n\). We could achieve this in our setting by choosing a prior parameterized by \(n\), such as a Laplace or Horseshoe prior with scale proportional to \(n^{-1/2}\).

Finally, we note that in~\cite{peters2014identifiability}, in the proof of theorem 1 it is directly shown that in the equal variance case, the solution to the problem in equation~\eqref{eq:5} is uniquely the ground truth set of parameters. So under the equal variance assumption and suitable sparsity priors, the posterior will concentrate to the ground truth parameters.

\section{Computational Use}
While developing this work, we estimate we used a total of around 1000 GPU-hours on a Nvidia 2080Ti, on an internal cluster, leading to emissions of around 100kg CO2 equivalent \footnote{Using \url{https://mlco2.github.io/impact }}.

\end{document}